\documentclass[a4paper, oneside]{article} %IEEE TEVC DECOMMENT
\usepackage{geometry}
\usepackage{afterpage}

\usepackage{pdflscape}
\usepackage{multirow}
\usepackage{amsthm,amsmath,graphicx, amssymb, graphics,  makeidx }
\usepackage{pdfsync}
\usepackage[algo2e,ruled,vlined,linesnumbered]{algorithm2e}
\usepackage{tikz}
\usetikzlibrary{automata,arrows,positioning,calc}
\usepackage{wrapfig,epstopdf}
\usepackage[nodisplayskipstretch]{setspace}
\usepackage{makeidx}
\makeindex
\usepackage{thmtools,thm-restate}
\DeclareMathOperator*{\argmax}{arg\,max}
\newtheorem{theorem}             {Theorem}

\newcommand{\mathsym}[1]{{}}
\newcommand{\unicode}[1]{{}}
\newcommand*{\muGA}{($\mu$+1)~GA\xspace}
\newcommand*{\GAtwo}{(2+1)~GA\xspace}
\newcommand*{\suga}{$(2+1)_{S}$~GA\xspace}
\newcommand{\om}{\text{\sc OneMax}\xspace}

\newcommand*{\ea}{(1+1)~EA\xspace}
\newcommand{\onemax}{\text{\sc OneMax}\xspace}
\newcommand{\jump}{\text{\sc Jump}\xspace}
\newcommand{\bo}{\ensuremath{{\mathcal{O}}}}

\graphicspath{{./Experiments/}}

\begin{document}

\title{Standard Steady State Genetic Algorithms Can Hillclimb Faster than 
Mutation-only Evolutionary Algorithms}

%IEEE TEVC : decomment
%
%\author{Dogan Corus,
        %and~Pietro~S.~Oliveto % <-this % stops a space
%\thanks{{D. Corus and P. S. Oliveto are with the Department of Computer Science, University of Sheffield,
%Regent Court, 211 Portobello, S1 4DP, Sheffield, UK. 
%e-mail: \{d.corus,p.oliveto\}@sheffield.ac.uk.} .}}% <-this % stops a space
%

%\markboth{IEEE TRANSACTIONS ON EVOLUTIONARY COMPUTATION, AUTHOR-PREPARED MANUSCRIPT}
%{IEEE TRANSACTIONS ON EVOLUTIONARY COMPUTATION, AUTHOR-PREPARED MANUSCRIPT}
%\fi

%arxiv : decomment

\author{Dogan Corus\\
\small Department of Computer Science,\\
\small University of Sheffield,\\
\small  Sheffield, UK
\and Pietro S. Oliveto\\%\\\\\\
\small Department of Computer Science,\\
\small University of Sheffield,\\
\small  Sheffield, UK
}

\maketitle

\begin{abstract}
Explaining to what extent the real power of genetic algorithms lies in the ability 
of crossover to recombine individuals into higher quality solutions is an 
important problem in evolutionary computation.
In this paper we show how the interplay between mutation and crossover can  make genetic algorithms hillclimb faster than their 
mutation-only 
counterparts. 
We devise a Markov Chain 
framework that allows  
to rigorously prove an upper bound on the runtime of standard steady state genetic algorithms
to hillclimb the \onemax 
function. The bound establishes that the steady-state genetic algorithms are 25$\%$
faster than all standard bit mutation-only evolutionary algorithms with static mutation rate up to lower order terms 
for moderate population sizes. 
The analysis also suggests that larger populations may be faster than 
populations of size 2.
We present a lower bound
for a greedy (2+1)~GA  
that matches the upper bound for populations larger than 2, 
rigorously proving that 2 individuals cannot outperform larger population sizes
under greedy selection and greedy crossover up to lower order terms. In complementary experiments the 
best population size is greater than 2 and the greedy genetic algorithms are faster than 
standard ones, further suggesting that the derived lower bound also holds for the 
standard steady state (2+1)~GA.

% Explaining that the real power of Genetic Algorithms (GAs) lies in the ability 
% of crossover to recombine individuals into higher quality solutions is an 
% important problem in evolutionary computation. Recently it has been shown how 
% GAs can efficiently escape local optima by creating the necessary diversity 
% through an interplay of mutation and crossover. In this paper we show how such 
% an interplay can also make GAs hillclimb faster than their mutation-only 
% counterparts.  We first present a mathematical framework that allows to obtain 
% upper bounds on the runtime of standard steady state GAs. We use the framework 
% to rigorously prove an upper bound of $3/4 e n \log n + O(n)$ on the runtime of 
% ($\mu$+1)~GAs with constant population size $\mu>2$ to hillclimb the \onemax 
% function, thus outperforming all mutation-only ($\mu$+$\lambda$)~EAs. By 
% providing a larger upper bound for $\mu=2$, our result suggests that population 
% sizes $\mu \geq 3$ may be faster. We present a matching lower bound of $3/4 e n 
% \log n + O(n)$ for a greedy (2+1)~GA that always selects parents with highest 
% fitness. This rigorously proves that $\mu=2$ cannot outperform $\mu \geq 3$ 
% under greedy selection. Complementary experimental results confirm that the 
% optimal population size is greater than 2 and that greedy GAs are faster than 
% standard ones, further suggesting that the lower bound also holds for the 
% standard steady state (2+1)~GA. 
\end{abstract}

\section{Introduction}

\label{sec:int}

\makeatletter{}%It is generally believed that the power of 
Genetic algorithms 
(GAs) rely on a  population of individuals that simultaneously explore the 
search space. The main distinguishing features of GAs from other randomised search heuristics is their use of a population and crossover to generate new solutions.
Rather than slightly modifying the current best solution as in more traditional heuristics,
the idea behind GAs is that new solutions are 
generated by recombining individuals of the current population (i.e., 
crossover). Such individuals are selected to reproduce probabilistically 
according to their fitness (i.e., reproduction).
Occasionally, random mutations may slightly modify the offspring produced by crossover.
The original motivation behind these mutations is to avoid that some genetic material may be lost forever,
thus allowing to avoid premature convergence~\cite{GOLDBERG,EIBENSMITH}.
%
%and in the capability of the crossover operator of combining 
%diverse solutions to generate higher quality offspring. %The role of 
%Mutation, 
%instead, is %generally relegated to that of  `diversity creator' 
%a relatively rare event that allows to avoid 
%premature convergence. In particular, it is crossover that {\it drives} the 
%optimisation while mutation contributes infrequently via small mutation rates \cite{GOLDBERG,ROWE}.
For these reasons the GA community traditionally regards crossover as the main search operator while mutation is considered a {\it ``background operator''} \cite{EIBENSMITH, BACK, HOLLAND}
or a {\it ``secondary mechanism of genetic adaptation''} \cite{GOLDBERG}.  
%Note, that it has also been rigorously proven for generational evolutionary 
%algorithms including generational GAs that the lower the selection pressure, 
%the lower the mutation rate should 
%be \cite{LehrePPSN2014,LehreGECCO2014,LehreGECCO2011}.

%Justifying these beliefs and 
Explaining when and why GAs are effective has 
proved to be a non-trivial task. Schema theory and its resulting building block 
hypothesis \cite{GOLDBERG} were devised to explain such working principles. 
However, these theories did not allow to rigorously characterise the behaviour 
and performance of GAs. The hypothesis was disputed when a class of functions 
(i.e., Royal Road), thought to be ideal for GAs, was designed and experiments 
revealed that the simple \ea was more efficient \cite{MitchellForrestFOGA1993, 
Jansen2005}. 

Runtime analysis approaches have provided rigorous proofs that crossover may 
indeed speed up the evolutionary process of GAs in ideal conditions (i.e., if 
sufficient diversity is available in the population). 
The \jump function was introduced by Jansen and Wegener as a first example where 
crossover considerably improves the expected runtime compared to mutation-only 
Evolutionary Algorithms (EAs) \cite{Jansen2002}. The proof required an 
unrealistically small crossover probability to allow mutation alone to create 
the necessary population diversity for the crossover operator to then escape the 
local optimum. Dang \emph{et al.} recently showed that the sufficient diversity, 
and even faster upper bounds on the runtime for not too large jump gaps, 
can be achieved also for realistic crossover probabilities by using diversity 
mechanisms \cite{Dang2016a}. Further examples that show the effectiveness of 
crossover have been given for both artificially constructed functions and 
standard combinatorial optimisation problems (see the next section for an 
overview).  

Excellent hillclimbing  performance of crossover based GAs has been 
also proved. B. Doerr et al. proposed a (1+($\lambda$,$\lambda$))~GA which 
optimises the \onemax function in \linebreak $\Theta(n \sqrt{\log \log \log 
(n)/\log\log(n)} )$ fitness evaluations (i.e., runtime) 
\cite{Doerr2015}, \cite{doerr_black-box_2015}. Since the unbiased unary black 
box complexity of 
\onemax is $\Omega(n \log n)$ \cite{LehreWitt2012}, the algorithm is 
asymptotically faster than any unbiased mutation-only evolutionary algorithm 
(EA). 
Furthermore, the algorithm runs in linear time when the population size is self-adapted throughout the run \cite{DoerrAdaptiveLambda}.
Through this work, though, it is hard to derive conclusions on the working 
principles of standard GAs because these are very different compared to the 
(1+($\lambda$,$\lambda$))~GA in several aspects. In particular, the (1+($\lambda$,$\lambda$))~GA 
was especially designed to use crossover as a repair mechanism that follows the creation of new solutions via high mutation rates.
This makes the algorithm work in a considerably different way compared to traditional GAs.

More traditional GAs have been analysed by Sudholt \cite{Sudholt2015}. 
Concerning \onemax, he shows how ($\mu$+$\lambda$)~GAs are twice as fast as 
their standard bit mutation-only counterparts. As a consequence, he showed an 
upper bound of 
$(e/2) n \log n (1 + o(1))$ function evaluations for a (2+1)~GA versus the $e n 
\log n (1 -o(1))$ function evaluations required by any standard bit mutation-only EA 
\cite{SudholtTEVC2013,Witt2013}. 
This bound further reduces to $1.19 n \ln n \pm O(n \log \log n)$ if the optimal mutation rate is used (i.e., $(1+\sqrt{5})/2 \cdot 1/n \approx 1.618/n$). 
However, the analysis requires that diversity is 
artificially enforced in the population by breaking ties always preferring 
genotypically different individuals. This mechanism ensures that once diversity 
is created on a given fitness level, it will never be lost unless a better 
fitness level is reached, giving ample opportunities for crossover to exploit 
this diversity.

Recently, it has been shown that it is not necessary to enforce diversity for 
standard steady state GAs to outperform standard bit mutation-only EAs 
\cite{Dang2016b}. In 
particular, the \jump function was used as an example to show how the interplay 
between crossover and mutation may be sufficient for the emergence of the 
necessary diversity to escape from local optima more quickly. 
Essentially, a runtime of $O(n^{k-1})$ may be achieved for any sublinear jump length $k>2$
versus the $\Theta(n^k)$ function evaluations required by standard bit  mutation-only EAs.

In this paper, we show that this interplay between mutation and crossover may 
also speed-up the hillclimbing capabilities of steady state GAs without the need 
of enforcing diversity artificially. In particular, we consider a standard 
steady state \muGA \cite{SarmaDeJongHANDBOOK,EIBENSMITH,ROWE} and prove an upper 
bound on the runtime  to hillclimb the \onemax function of $(3/4) e n \log n + 
O(n)$ for any $ \mu \geq 3$  and $\mu= o(\log{n}/\log{\log{n}})$ when the 
standard $1/n$ mutation rate is 
used. %We first derive a mathematical framework based on Markov chain theory that 
%allows us to derive upper bounds on the runtime of standard steady state GAs. 
Apart from showing that standard ($\mu$+1)~GAs are faster than their standard 
bit
mutation-only counterparts up to population sizes 
$\mu=o(\log{n}/\log{\log{n}})$, the framework provides two other interesting 
insights. Firstly, it delivers better runtime bounds for mutation rates that are 
higher than the standard $1/n$ rate. The best upper bound of $0.72 e n \log n + 
O(n)$ is achieved for $c/n$ with $c=\frac{1}{2} \left(\sqrt{13}-1\right) \approx 
1.3$. Secondly, the framework provides a larger upper bound, up to lower order terms, for the (2+1)~GA 
compared to that of any $\mu \geq 3$ and $\mu=o(\log{n}/\log{\log{n}})$. The 
reason for the larger constant in the leading term of the runtime 
is that, for populations of size 2, there is always a constant probability that 
any selected individual takes over the population in the next generation. This 
is not the case for population sizes larger than 2.

To shed light on the exact runtime for population size $\mu=2$ we present a 
lower bound analysis for a greedy genetic algorithm, which we call (2+1)$_S$~GA, that always selects individuals of 
highest fitness for crossover and always successfully recombines them if their Hamming distance is greater than 2. 
This algorithm is similar to the one analysed 
by Sudholt \cite{Sudholt2015} to allow the derivation of a lower bound, with the 
exception that we do not enforce any diversity artificially and that our crossover operator is slightly less greedy 
(i.e., in \cite{Sudholt2015} crossover always recombines correctly individuals also when the Hamming distance is exactly 2). Our analysis 
delivers a matching lower bound for all mutation rates $c/n$, where $c$ is a 
constant, for the 
greedy (2+1)$_S$~GA (thus also $(3/4) e n \log n + O(n)$ and $0.72 e n \log n + 
O(n)$ respectively for mutation rates $1/n$ and $1.3/n$). This result 
rigorously 
proves that, under greedy selection and semi-greedy crossover, the (2+1)~GA cannot outperform any \muGA 
with $\mu \geq 3$ and $\mu = o(\log{n}/\log{\log{n}})$. 

We present some experimental investigations to shed light on the questions that 
emerge from the theoretical work. 
In the experiments we consider the commonly used parent selection that chooses uniformly at random from the population 
with replacement (i.e., our theoretical upper bounds hold for a larger variety of parent selection operators).
  We first compare the performance of the 
standard steady state GAs against the fastest standard bit mutation-only EA with fixed mutation rate (i.e., the \ea \cite{SudholtTEVC2013,Witt2013})
and the GAs that have been proved to outperform it. The experiments show that 
the speedups over the \ea occur already for small problem sizes $n$ and that 
population sizes larger than $2$ are faster than the standard (2+1)~GA. 
Furthermore, the greedy \suga indeed appears to be faster than the standard 
(2+1)~GA\footnote{We thank an anonymous reviewer for pointing out that this is not obvious.}, further suggesting that the theoretical lower bound also holds for the 
latter algorithm. Finally, experiments confirm that larger mutation rates than 
$1/n$ are more efficient. In particular, better runtimes are achieved for 
mutation rates that are even larger than the ones that minimise our theoretical 
upper bound (i.e., $c/n$ with 1.5 $ \leq c \leq$ 1.6 versus the $c=$1.3 we have 
derived mathematically; interestingly this experimental rate is similar to the 
optimal mutation rate for OneMax of the algorithm analysed in \cite{Sudholt2015}). These 
theoretical and experimental results seem to be 
in line with those recently presented for the same steady state GAs for the 
\jump function \cite{Dang2016b, Dang2016a}: higher mutation rates than $1/n$ are 
also more effective on \jump. %In particular, the results presented in 
%\cite{Dang2016b, Dang2016a} indicate that the greater the length of the \jump, 
%the higher is the best mutation rate (i.e., approximately $2/n$ for \jump length 
%$k=2$ and approximately $2.6/n$ for $k=4$): \onemax presents the minimal \jump 
%length of $k=1$.

The rest of the paper is structured as follows.
% 
%
%Since our 
%mathematical framework gives a weaker upper bound for $\mu=2$ we perform some 
%experiments which indeed suggest lower expected runtimes for constant $\mu>2$.
%
In the next section we briefly review previous related works that consider 
algorithms using crossover operators. In Section \ref{sec:preliminaries} we give 
precise definitions of the steady state \muGA and of the  \onemax function. In 
Section \ref{sec:framework} we present the Markov Chain framework that we will use for the 
analysis of steady state elitist GAs. In Section \ref{sec:upperbound} we apply 
the framework to analyse the \muGA and present the upper bound on the runtime 
for any   $3 \leq \mu=o(\log{n}/\log{\log{n}})$ and mutation 
rate $c/n$ for any constant $c$. In Section 
\ref{sec:lowerbound} we present the matching lower bound on the runtime of the 
greedy (2+1)$_S$~GA. In Section \ref{sec:experiments} we present our 
experimental findings. In 
the Conclusion we present a discussion and open questions for future work.

\section{Related Work}
The first rigorous groundbreaking proof that crossover can considerably improve 
the performance of EAs was given by Jansen and Wegener 
 for the \muGA with an unrealistically low crossover probability \cite{Jansen2002}. 
A series of following works on the analysis of the \jump function have made the 
algorithm characteristics increasingly realistic 
\cite{Dang2016a,kotzing_how_2011}.
Today it has been rigorously proved that the standard steady state \muGA with 
realistic parameter settings does not require artificial
diversity enforcement to outperform its standard bit mutation-only counterpart 
to escape the 
plateau of local optima of the \jump function \cite{Dang2016b}.

Proofs that crossover may make a difference between polynomial and exponential 
time for escaping local optima have also been available for some time 
\cite{Storch2004,Jansen2005}. The authors devised example functions 
where, if sufficient diversity was enforced by some mechanism, then crossover 
could efficiently combine different individuals into an optimal solution. 
Mutation, on the other hand required a long time because of the great Hamming 
distance between the local and global optima. The authors chose to call the 
artificially designed functions {\it Real Royal Road} functions because the 
Royal Road functions devised to support the building block hypothesis had failed 
to do so \cite{MitchellHollandForrest94}. The Real Royal Road functions, though, 
had no resemblance with the schemata structures required by the building block 
hypothesis.

The utility of crossover has also been proved for less artificial problems such 
as coloring problems inspired by the Ising model from physics 
\cite{Sudholt2005}, computing input-output sequences in finite state machines 
\cite{Lehre2011},
shortest path problems \cite{Doerr2012},  vertex cover \cite{Neumann2011} and
multi-objective optimization problems \cite{Qian2013}.
%
%With the exception of \cite{Doerr2012}, 
The above works show that crossover 
allows to escape from local optima that have large basins of attraction for the 
mutation operator. Hence, they establish the usefulness of crossover as an 
operator to enchance the exploration capabilities of the algorithm. 

%Nevertheless, they do not show how 
The interplay between crossover and mutation 
may produce a speed-up also in the exploitation phase, for instance when the 
algorithm is hillclimbing.
Research in this direction has recently appeared. The  design of the 
(1+($\lambda,\lambda$))~GA was theoretically driven to beat the $\Omega(n \ln 
n)$ lower bound of all unary unbiased black box algorithms. Since the dynamics of the 
algorithm differ considerably from those of standard GAs, it is difficult to 
achieve more general conclusions about the performance of GAs from the analysis 
of the (1+($\lambda,\lambda$))~GA. From this point of view the work of Sudholt 
is more revealing when he shows that any standard ($\mu+\lambda$)~GA outperforms 
its standard bit mutation-only counterpart for hillclimbing the \onemax 
function 
\cite{Sudholt2015}. The only caveat is that the selection stage enforces 
diversity artificially, similarly to how Jansen and Wegener had enforced 
diversity for the Real Royal Road function analysis. In this paper we rigorously 
prove that it is not necessary to enforce diversity artificially for 
standard-steady state GAs to outperform their standard bit mutation-only 
counterpart. 
\section{Preliminaries}
\label{sec:preliminaries}
\makeatletter{}
\begin{algorithm2e}[t] \label{alg:mu+1-GA}
    \caption{\muGA \cite{SarmaDeJongHANDBOOK,EIBENSMITH,ROWE,Dang2016b}
    }

    $P \gets \mu \textrm{ individuals, uniformly at random from } \{0, 1\}^n$\;
    \Repeat{\textrm{termination condition satisfied}}
    {
            Select $x, y \in P$ with replacement using an operator abiding (\ref{select})\;
            $z \gets$ Uniform crossover with probability $1/2$  $(x, y)$\;
            Flip each bit in $z$ with probability $c/n$\;  
             $P \gets P \cup \{z\}$\;
        Choose one element from $P$ with lowest fitness and remove it from $P$, 
breaking ties at random; 
    }
  \end{algorithm2e}

We will analyse the runtime (i.e., the expected number of fitness function 
evaluations before an optimal search point is found) of 
a steady state genetic algorithm with population size $\mu$ and offspring size 1 
(Algorithm~\ref{alg:mu+1-GA}). In steady state GAs the entire population is not 
changed at once, but rather a part of it. In this paper we consider the most 
common option of creating one new solution per generation 
\cite{SarmaDeJongHANDBOOK,ROWE}. Rather than restricting the algorithm to the most 
commonly used uniform selection of two parents, we allow more flexibility to the 
choice of which parent selection mechanism is used. This approach was also 
followed by Sudholt for the analysis of the \muGA with diversity 
\cite{Sudholt2015}.  In each generation the algorithm picks two parents from its 
population with replacement using a selection operator that satisfies the following condition.
\begin{equation}
\label{select}
\forall x,y: \; f(x)\geq f(y) \implies \Pr(\text{select }x) \geq 
\Pr(\text{select }y).
\end{equation}

The condition allows to use most of the popular parent selection mechanisms with replacement
such as fitness proportional selection, rank selection or the one commonly used 
in steady state GAs, i.e., uniform selection \cite{EIBENSMITH}.
Afterwards, uniform crossover between the selected parents (i.e., each bit of 
the offspring is chosen from each parent with probability $1/2$) provides  an 
offspring to which standard bit mutation (i.e., each bit is flipped with with 
probability $c/n$) is applied. The best $\mu$ among the $\mu+1$ solutions are 
carried over to the next generation and ties are broken uniformly at random.

In the paper we use the standard convention for naming steady state algorithms: the ($\mu$+1)~EA differs from the \muGA 
by only selecting one individual per generation for reproduction and applying standard bit mutation to it (i.e., no crossover).
Otherwise the two algorithms are identical.

We will analyse Algorithm~\ref{alg:mu+1-GA} for the well-studied $\om$
function that is defined on bitstrings $x\in \{0,1\}^n$ of 
length $n$ and returns the number of $1$-bits in the string: 
$ \om(x)=\sum_{i=1}^{n}x_i$.
Here $x_i$ is the $i$th bit of the solution $x\in \{0,1\}^n$. 
The \onemax benchmark function is very useful to assess the hillclimbing 
capabilities of a search heuristic. 
It displays the characteristic function optimisation property that finding improving solutions
becomes harder as the algorithm approaches the optimum.
%In this paper we show that the \muGA hillclimbs the $\om$ function faster than 
%its mutation-only EA counterparts.  
The problem is the same as that of identifying the hidden solution of the Mastermind game
where we assume for simplicity that the target string is the one of all 1-bits.
Any other target string $z\in \{0,1\}^n$ may also be used without loss of generality.
If a bitstring is used, then $\om$ is equivalent to Mastermind with two colours 
\cite{Doerr2014}.
This can be generalised to many colours if alphabets of greater size are used 
\cite{DoerrDoerrSpohelHenning2016, Doerr2016}.

\section{Markov Chain Framework}
\label{sec:framework}
\begin{figure}\caption{Markov Chain for fitness level $i$.} 
\label{fig-mchain}
% \makeatletter{}
\begin{center}
\begin{tikzpicture}[->, >=stealth', auto, semithick, node distance=3cm]
\tikzstyle{every state}=[fill=white,draw=black,thick,text=black,scale=1]
 \node[state]    (B){$S_{1,i}$};
\node[state]    (C)[right of=B]   {$S_{2,i}$};
 \node[state, double]    (D)[right of=C]   {$S_{3,i}$};
 \path
 (B) edge[loop below]     node{$1-p_m-p_d$}         (B);
 \path
 (B) edge[bend left, below]     node{$p_d$}         (C);
  \path
 (C) edge[loop below]     node{$1-p_c-p_r$}         (C);
 \path
 (C) edge[bend left, below]     node{$p_r$}         (B);
 \path
 (C) edge[right, below]     node{$p_c$}         (D);
 \path
 (B) edge[bend left, above]     node{$p_m$}         (D);
 \end{tikzpicture}
\end{center}
 \end{figure}
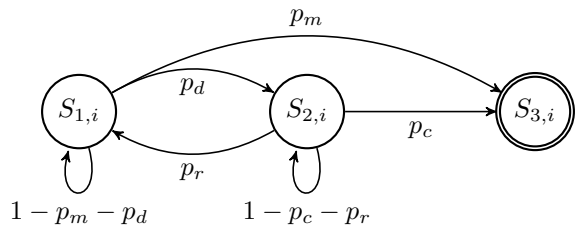
% \makeatletter{}
The recent analysis of the \muGA for the \jump function shows that the interplay between crossover and mutation may
create the diversity required for crossover to decrease the expected time to jump towards the optimum \cite{Dang2016b}. 
At the heart of the proof is the analysis of a random walk on the number of diverse individuals on the local optima of the function.
The analysis delivers improved asymptotic expected runtimes of the  \muGA over mutation-only EAs only for population sizes $\mu=\omega(1)$.
This happens because, for larger population sizes, it takes more time to lose 
diversity once created, hence crossover has more time to exploit it.
For \onemax the technique delivers worse asymptotic bounds for population sizes $\mu=\omega(1)$ and an $O(n \ln n)$ bound for constant population size.
Hence, the techniques of \cite{Dang2016b} cannot be directly applied to show a speed-up of the \muGA over mutation-only EAs and a careful analysis of the leading constant in the runtime is necessary.
In this section we present the Markov chain framework that we will use to obtain 
the upper bounds on the runtime of the elitist steady state GAs. 
We will afterwards discuss how this approach builds upon and generalises Sudholt's approach in \cite{Sudholt2015}.

The $\om$ function has $n+1$ distinct fitness values. We divide the search 
space 
into the following canonical fitness levels \cite{JANSENBOOK, OlivetoYao2010}: 
\[L_i=\{x\in \{0,1\}^n|\om(x)=i\}.\] 
We say that 
a population is in fitness level $i$ if and only if its best solution is in 
level $L_i$.

We use a Markov chain (MC) for each fitness level $i$ to represent the different 
states the population may be in before reaching the next fitness level. The MC 
depicted in Fig.~\ref{fig-mchain} distinguishes between states where the 
population has no diversity (i.e., all individuals have the same genotype), 
hence crossover is ineffective, and states where diversity is available to be 
exploited by the crossover operator. The MC has one absorbing state  and two 
transient states. The first transient state $S_{1,i}$ is adopted if the whole 
population consists  of  copies of the same individual at level $i$ (i.e., all the individuals have the same genotype). The second 
state $S_{2,i}$ is reached if the population consists of $\mu$ individuals  in 
fitness level $i$ and at least two individuals $x$  and $y$ are not identical. The second transient state $S_{2,i}$ differs from the 
state $S_{1,i}$ in having diversity which can be exploited by the crossover 
operator.  $S_{1,i}$ and $S_{2,i}$ are mutually accessible from each other since 
the diversity can be introduced at state $S_{1,i}$ via mutation with some 
probability $p_d$ and can be lost at state $S_{2,i}$ with some relapse probability $p_r$ when 
copies  of  a solution take over the  population. 

% via mutation, when copies 
% of 
% a solution take over the 
% population or via crossover if different individuals are selected. 
 The absorbing state $S_{3,i}$ is reached when a solution 
at a better fitness level is found, an event that happens with probability 
$p_m$ 
when the population is at state $S_{1,i}$ and with probability $p_c$ when the 
population is at state $S_{2,i}$. 
We pessimistically assume that in $S_{2,i}$ 
there is always only one single individual with a 
different genotype (i.e., with more than one distinct individual, $p_c$ would be 
higher and $p_r$ would be zero). 
Formally when $S_{3,i}$ is reached the population is no longer in level $i$ because a better fitness level has
been found. However, we will bound  the expected time to reach the absorbing state for the next level only when the whole population
has reached it (or a higher level). We do this because we assume that initially all the population is in level $i$ when calculating the transition probabilities in the MC for each level $i$. This implies that 
bounding the expected times to reach the absorbing states of each fitness level 
is not sufficient to achieve an upper bound on the total expected runtime. When 
$S_{3,i}$ is reached for the first time, the population only has one individual 
at the next fitness level or in a higher one.  Only when all the individuals have reached level 
$i+1$ (i.e., either in state $S_{1,i+1}$ or $S_{ 2,i+1}$) may we use the MC to 
bound the runtime to overcome level $i+1$.  Then the MC can be applied, 
once per fitness level, to bound the total runtime  until the optimum is 
reached.

%The MC is depicted in Fig. \ref{fig-mchain}.

The main distinguishing aspect between the analysis presented herein and that 
of 
Sudholt \cite{Sudholt2015}
is that we take into account the possibility to transition back and forth  
(i.e., resp. with probability $p_d$ and $p_r$) between states $S_{1,i}$ and 
$S_{2,i}$ as in standard steady state GAs (see Fig. \ref{fig-mchain}).
By enforcing that different genotypes on the same fitness level are kept in the population, 
the genetic algorithm considered in \cite{Sudholt2015} has a good probability of exploiting this diversity to recombine the different individuals.
In particular, once the diversity is created it will never be lost, giving many opportunities for crossover to take advantage of it.
A crucial aspect is that the probability of increasing the number of ones via crossover is much higher 
than the probability of doing so via mutation once many 1-bits have been collected.
Hence, by enforcing that once State $S_{2,i}$ is reached it cannot be left until a higher fitness level is found, Sudholt could prove that the resulting algorithm 
is faster compared to only using standard bit mutation. 
In the standard steady state GA, instead, once the diversity is created it may subsequently be lost before crossover successfully recombines the diverse individuals.
This behaviour is modelled in the MC by considering the relapse probability $p_r$. 
Hence, the algorithm spends less time in state $S_{2,i}$ compared to the GA with diversity enforcement. 
Nevertheless, it will still spend some optimisation time in state $S_{2,i}$ where it will have a higher probability of improving its fitness by exploiting the diversity via crossover
than when in state $S_{1,i}$ (i.e., no diversity) where it has to rely on mutation only. For this reason the algorithm will not be as fast for \om as the GA with enforced diversity 
but will still be faster than standard bit mutation-only EAs. 

An interesting consequence of the possibility of losing diversity, is that 
populations of size greater than 2 can be beneficial.
In particular the diversity (i.e., State $S_{2,i}$) may be completely lost in the next step when there is only one diverse individual left in the population.
When this is the case, the relapse probability $p_r$ decreases with the population size $\mu$ because the probability of selecting the diverse individual for removal is $1/\mu$.
Furthermore, for population size $\mu=2$ there is a positive probability that diversity is lost in every generation by either of the two individuals taking over, while for larger population sizes this is not the case. 
As a result our MC framework analysis will deliver a better upper bound for $\mu>2$ compared to the bound for $\mu=2$.
This interesting insight into the utility of larger populations could not be seen in the analysis of \cite{Sudholt2015} because there, once the 
diversity is achieved, it cannot be lost.

We first concentrate on the expected absorbing time of the MC. 
Afterwards we will calculate the takeover time before we can transition from one MC to the next.
Since it is not easy to derive the exact transition probabilities, a runtime 
analysis is considerably simplified by using bounds on these
probabilities.
The main result of this section is stated in the following theorem that shows that we can  use lower bounds on 
the transition probabilities moving in the direction of the absorbing state 
(i.e., $p_m$, $p_d$ and $p_c$) and an upper bound on the probability of moving 
in 
the opposite direction to no diversity (i.e., $p_r$) to derive an upper bound 
on the expected absorbing time of the Markov chain. In particular, we define a 
Markov chain $M'$ that uses the bounds on the exact transition probabilities and 
show that its expected absorbing time is greater than the absorbing time of the 
original chain.
Hereafter, we drop the level index $i$ for brevity and use $E[T_1]$ and 
$E[T_2]$ instead of $E[T_{1,i}]$ and $E[T_{2,i}]$ (Similarly, $S_{1}$ will 
denote state $S_{1,i}$).

\begin{theorem}
\label{lem:mcdom}
Consider two  Markov chains $M$ and $M'$ with the topology in 
Figure~\ref{fig-mchain} where the transition probabilities for $M$ are $p_c$, 
$p_m$, $p_d$ , $p_r$ and the transition probabilities for $M'$ are $p_c'$, 
$p_m'$, $p_d'$ and $p_r'$. Let the expected absorbing time for $M$ be $E[T]$
and the expected absorbing time of $M'$ starting from state $S_{1}$
be $E[T_{1}']$ respectively. If 
\begin{itemize}
 \item $p_m<p_c$
 \item $p_d' \leq p_d$
 \item $p_r' \geq p_r$
 \item $p_c' \leq p_c$
 \item $p_m' \leq p_m$
\end{itemize}
Then $E[T]\leq E[T_{1}'] \leq 
\frac{p_c'+p_r'}{p_c' p_d' + p_c' p_m' + p_m' p_r'}+\frac{1}{p_c'}$.
\end{theorem}

We first concentrate on the second inequality in the statement of the theorem which will follow immediately from the 
next lemma. It allows us to obtain the expected absorbing time of the MC if 
the exact values for the transition probabilities are known.
In particular, the lemma establishes the expected 
times $E[T_{1}]$ and $E[T_{2}]$ to reach the absorbing state, starting from 
the states $S_{1}$ and $S_{2}$ respectively. 
\begin{restatable}{lemma}{lemchain}
\label{lem:mchain}
 The expected times $E[T_{1}]$ and $E[T_{2}]$ to reach the absorbing state, 
starting  from state $S_{1}$ and  $S_{2}$ respectively are as follows:
  \begin{align*}
    E[T_{1}] &= \frac{p_c+p_r+p_d}{p_c p_d + p_c p_m + p_m p_r}\leq 
\frac{p_c+p_r}{p_c p_d + p_c p_m + p_m p_r}+\frac{1}{p_c}\\
    E[T_{2}] &= \frac{p_m+p_r+p_d}{p_c p_d + p_c p_m + p_m p_r}.
  \end{align*}
\end{restatable}
\begin{proof}
We analyse the MC and using the law of total expectation together with the 
conditional probabilities we establish the following recurrence equations:
\begin{align*}
 E[T_1]&= (E[T_2]+1)  p_d +  p_m+ 
(1+ E[T_1]) (1- p_d- p_m) \\
E[T_2]&=  ( E[T_1]+1)p_r
+  p_c+(1+ E[T_2]) (1- p_c-p_r ).
\end{align*}

We start by solving the system of equations for the Markov chain.  In order to 
get an expression for $E[T_1]$, we will first express $E[T_2]$ in terms of 
$E[T_1]$.

\begin{align*}
& E[T_2] =  ( E[T_1]+1)p_r
+  p_c+(1+ E[T_2]) (1- p_c-p_r ), \\
%&\iff\\
\end{align*}
implying
\begin{align*}
& E[T_2]=\frac{(E[T_1]+1)p_r-p_r+1}{p_c+p_r} =\frac{p_r E[T_1] 
+1}{p_c+p_r}. 
\end{align*}

We now substitute the expression for $E[T_2]$ into the equation for $E[T_1]$:

\begin{align*}
 E[T_1] &= \left(\frac{p_r E[T_1] +1}{p_c+p_r}+1\right)  p_d +  p_m+ (1+ 
E[T_1]) 
(1- p_d- 
p_m).  \\
\end{align*}
%\iff 
Hence,
\begin{align*}
& E[T_1] =\frac{p_c+p_d+p_r}{p_c p_d + p_c p_m + p_m p_r}.
\end{align*}

The expression for $E[T_1]$ can be bounded from above by 
separating the $p_d$ term in the numerator:

\begin{align*}
 E[T_1]& =\frac{p_c+p_r}{p_c p_d + p_c p_m + p_m p_r}+\frac{p_d}{p_c p_d 
+ p_c p_m + p_m p_r}\\
    &\leq \frac{p_c+p_r}{p_c p_d + p_c p_m + p_m p_r}+\frac{p_d}{p_c p_d} 
    \\ &\leq \frac{p_c+p_r}{p_c p_d + p_c p_m + p_m p_r} +\frac{1}{p_c}.
\end{align*}

If we substitute the value of $E[T_1]$ in the above expression for $E[T_2]$ we 
obtain:
\begin{align*}
 E[T_2]&=\frac{\frac{p_r (p_c+p_d+p_r)}{p_c 
p_d+p_c p_m+p_m p_r}+1}{p_c+p_r}\\
 &= \frac{p_r (p_c+p_d+p_r)+p_c 
p_d+p_c p_m+p_m p_r}{(p_c+p_r) 
(p_c p_d+p_c p_m+p_m p_r)}\\
&= \frac{p_c (p_m+p_r+p_d)+p_r 
(p_m+p_r+p_d)}{(p_c+p_r) (p_c 
p_d+p_c p_m+p_m p_r)}\\
&=\frac{p_m+p_r+p_d}{p_c p_d+p_c 
p_m+p_m p_r}.
\end{align*}
\end{proof}
Before we prove the first inequality in the statement of Theorem \ref{lem:mcdom}, we will derive some helper propositions. We first show 
that as long as the transition probability of reaching the absorbing state from 
the state $S_2$ (with diversity) is greater than that of reaching the absorbing 
state from the state with no diversity $S_1$ (i.e., $p_m <p_c$), then the 
expected absorbing time from state $S_1$ is at least as large as the expected 
time unconditional of the starting point. This will allow us to achieve a 
correct upper bound on the runtime by just bounding the absorbing time from 
state $S_1$. In particular, it allows us to pessimistically assume that the 
algorithm starts each new fitness level in state $S_1$ (i.e., there is no 
diversity in the population).

 %Markov chain that uses the exact transition probabilities.
%The 

%We now establish the following straightforward characteristics 
%of the MC that will allow the use of bounds on the transition probabilities. 

\begin{restatable}{proposition}{promcone}\label{pro:MC1}
  Consider a Markov chain with the topology given in Figure~\ref{fig-mchain}. 
Let $E[T_1]$ and $E[T_2]$ be the expected absorbing times starting from state 
$S_{1}$ and $S_{2}$ respectively. If $p_m< p_c$, then $E[T_1]>E[T_2]$ and 
$E[T]$, the unconditional expected absorbing time, satisfies $E[T]\leq 
E[T_1]$.
 \end{restatable}
 \begin{proof}
 From Lemma~\ref{lem:mchain},
  \begin{align*}
  E[T_1]&=\frac{p_c+p_d+p_r}{p_c p_d + p_c p_m + p_m p_r} \\
  E[T_2]&=\frac{p_m+p_d+p_r}{p_c p_d + p_c p_m + p_m p_r} .
  \end{align*}
Since the denominators in both expressions are the same, $E[T_1]>E[T_2]$ 
follows from $p_c+p_d+p_r > p_m+p_d+p_r$, which in turn follows from $p_c 
> p_m$. 
The unconditional expected absorbing time is calculated as the weighted sum 
$E[T]=p\cdot E[T_1] + (1-p) \cdot E[T_2]$ where  $p$ is the probability that 
the initial state is $S_1$ and $1-p$ is the probability that the initial state 
is $S_2$. Since $E[T_1]\geq E[T_2]$, the weighted sum $E[T]$ is also smaller 
than or equal to $E[T_1]$.
 \end{proof}
  
 In the following proposition we show that if we overestimate the probability of 
losing diversity and underestimate the probability of increasing it, then we 
achieve an upper bound on the expected absorbing time as long as $p_m < p_c$. 
Afterwards, in Proposition \ref{pro:MC3} we show that an upper bound on the 
absorbing time is also achieved if the probabilities $p_c$ and $p_m$ are 
underestimated.

\begin{restatable}{proposition}{promctwo}\label{pro:MC2}
 Consider two  Markov chains $M$ and $M'$ with the topology in 
Figure~\ref{fig-mchain} where the transition probabilities for $M$ are $p_c$, 
$p_m$, $p_d$, $p_r$ and the transition probabilities for $M'$ are $p_c$, 
$p_m$, $p_d'$ and $p_{r}'$. Let the expected absorbing times starting from 
state $S_{1}$ for $M$ and $M'$ 
be $E[T_1]$ and $E[T_{1}']$ respectively.  If $p_d'\leq p_d$,  $p_r' \geq p_r$ 
and $p_m < p_c$, then $E[T_1] \leq E[T_{1}']$.
\end{restatable}
\begin{proof}
   Let $r$ and $d$ be non-negative slack variables such that $p_d' = p_d-d$,  
$p_r' = p_r+r$.
 We prove the claim that the absorbing times
 \begin{align*}
  E[T_1]&=\frac{p_c+p_d+p_r}{p_c p_d + p_c p_m + p_m p_r}, \\
  E[T_{1}']&=\frac{p_c+(p_d-d)+(p_r+r)}{p_c (p_d-d) + p_c p_m + p_m (p_r+r)}, 
  \end{align*}
  satisfy
  \begin{align*}
   E[T_{1}']-E[T_1]\geq 0.
  \end{align*}
For readability purposes let $A=p_c+p_d+p_r$ and $B=p_c p_d + p_c p_m + p_m 
p_r$. Then,

\begin{align*}
E[T_{1}']&=\frac{p_c+p_d+p_r -d + r}
					{p_c p_d + p_c p_m + p_m 
p_r- p_c d+p_m r}\\
&=\frac{A+ (r - d)}
					{B- p_c d+p_m r},\\
%\end{align*}
%
%\begin{align*}
E[T_{1}']-E[T_1]&=\frac{A+(r-d)}{B-p_c d+p_m r}-\frac{A}{B}\\ &=\frac{B r-B 
d+A p_c d-A p_m r}{B (B-p_c d+p_m r)}.
\end{align*}
Since the denominator is the product of the denominators of 
$E[T_{1}']$ and $E[T_1]$, we already know that it is positive.  
We now show that:
\begin{equation*}
 B r-B d+A p_c d-A p_m r \geq 0.
\end{equation*}
If we insert the values of $A$ and $B$ we obtain:
\begin{align*}
  B r-B d+A p_c d-A p_m r  &=    ( p_c p_d + p_c p_m + p_m p_r) r -( p_c p_d + p_c p_m + p_m p_r) d\\
&+(p_c+p_d+p_r) p_c d-(p_c+p_d+p_r) p_m r\\
%&(p_c+p_d+p_r) d p_c - (p_c+p_d+p_r) p_m r \\ 
%& -  (p_c p_d + p_c p_m + p_m 
%p_r) d  +( p_c p_d + p_c p_m + p_m p_r) r\\
=& p_c p_d r -p_c p_m d - p_m p_r d + p_{c}^{2} d + p_r p_c d - p_d p_m r\\
%=&d p_c^2-d p_c p_m +d p_c p_r-d p_m p_r +p_c p_d r-p_d p_m r\\
=& p_d r (p_c-p_m) +  p_c d (p_c-p_m)+ p_r d  (p_c-p_m).
%& A d p_c-A p_m r-B d+B r  = \\
%&(p_c+p_d+p_r) d p_c - (p_c+p_d+p_r) p_m r \\ 
%& -  (p_c p_d + p_c p_m + p_m 
%p_r) d  +( p_c p_d + p_c p_m + p_m p_r) r\\
%=&d p_c^2-d p_c p_m +d p_c p_r-d p_m p_r +p_c p_d r-p_d p_m r\\
%=&d p_c (p_c-p_m)+ d p_r (p_c-p_m)+p_d r (p_c-p_m)
\end{align*}
According to our assumption $p_c-p_m > 0$ the proposition follows because the 
probabilities and slack variables are non-negative. 
\end{proof}

\begin{restatable}{proposition}{promcthree}
\label{pro:MC3}
Consider two  Markov chains $M$ and $M'$ with the topology in 
Figure~\ref{fig-mchain} where the transition probabilities for $M$ are $p_c$, 
$p_m$, $p_d$, $p_r$ and the transition probabilities for $M'$ are $p_c'$, 
$p_m'$, $p_d$ and $p_r$. Let the expected absorbing times starting from 
state $S_{1}$ for $M$ and $M'$ be $E[T_1]$ and $E[T_{1}']$ respectively.  If 
$p_c'\leq p_c$ and  $p_m' \leq p_m$, then $E[T_1] \leq E[T_{1}']$.
\end{restatable}
\begin{proof}
 Let $c$ and $m$ be non-negative slack variables such that $p_c' + c= p_c$,  
$p_m' +m= p_m$. Similarly to the proof of Proposition~\ref{pro:MC2}, we prove 
the claim that the absorbing times
 \begin{align*}
  E[T_{1}']&=\frac{p_c'+p_d+p_r}{p_c' p_d + p_c' p_m' + p_m' p_r}, \\
  E[T_{1}]&=\frac{(p_c'+c)+p_d+p_r}{(p_c'+c)p_d + (p_c'+c) (p_m'+m) + (p_m'+m) 
p_r}, 
  \end{align*}
  satisfy
\begin{align*}
   E[T_{1}']-E[T_{1}]\geq 0.
\end{align*}
Again for readability purposes let $A=p_c'+p_d+p_r$ and $B=p_c' p_d + p_c' p_m' 
+ 
p_m' p_r$. Then,
\begin{align*}
 & E[T_{1}']-E[T_{1}]=\frac{A}{B}-\frac{A+c}{B+c p_d+ 
p_c' m+c p_m'+c m+m p_r}\\
=&\frac{A c p_d+ A p_c' m+ A c p_m'+ A c m+ Am p_r-B c}{B 
(B+c p_d+ p_c' m+c p_m'+c m+m p_r)}.
\end{align*}
Since the denominator is positive we focus on proving that the numerator ($N$) 
is also positive
\begin{align*}
N=A c p_d+ A p_c' m+ A c p_m'+ A c m+ Am p_r-B c \geq 0.
\end{align*}
Substituting the actual values for $A$ and $B$, we obtain the following 
equivalent expression:
\begin{align*}
N=& (p_c'+p_d+p_r) c p_d+ (p_c'+p_d+p_r) p_c' m\\ &+ (p_c'+p_d+p_r) c p_m'+ 
(p_c'+p_d+p_r) c m \\ &+ (p_c'+p_d+p_r) m p_r - (p_c' p_d + p_c' p_m' 
+ 
p_m' p_r) c \\
=& (p_d+p_r) c p_d+ (p_c'+p_d+p_r) p_c' m\\ &+ (p_d+p_r) c p_m'+ (p_c'+p_d) c m 
+ (p_c'+p_d+p_r) m p_r .
%N=& c m (p_c'+p_d+p_r)+c p_d (p_c'+p_d+p_r)\\&+c p_m' (p_c'+p_d+p_r)+m p_c' 
%(p_c'+p_d+p_r)\\ &+m p_r (p_c'+p_d+p_r)-c (p_c' p_d+p_c' p_m'+p_m' p_r)\\
%=&c m p_c'+c m p_d+c m p_r+c p_d^2 +c p_d p_m'+c p_d p_r\\ &+m p_c'^2+m p_c' 
%p_d+2 
%m 
%p_c' p_r+m p_d p_r+m p_r^2
\end{align*}
Since all of the above terms are positive the proposition follows.

\end{proof}

The propositions use that by lower bounding $p_d$ and upper bounding 
$p_r$ we overestimate the expected number of generations the 
population is in state $S_{1}$ compared to the time spent in state $S_{2}$. 
Hence, if $p_c > p_m$ we can safely use a lower bound for $p_d$ and 
an upper bound for $p_r$  and still obtain a valid upper bound on the runtime  
$E[T_{1}]$. This is rigorously shown by combining together 
the results of the previous propositions to prove the main result i.e., Theorem \ref{lem:mcdom}.

\begin{proof}[Proof of Theorem \ref{lem:mcdom}]
Consider a third Markov chain $M^{*}$ whose transition probabilities are $p_c$, 
$p_m$, $p_r'$, $p_d'$. Let the absorbing time of $M$ starting from state 
$S_{1}$ be $E[T_1]$. In order to prove the above statement we will prove the 
following sequence of inequalities.

\begin{equation*}
 E[T]\leq E[T_1]\leq E[T_{1}^{*}]\leq E[T_{1}'].
\end{equation*}

According to Proposition~\ref{pro:MC1},
$E[T]\leq E[T_1]$ since $p_c>p_m$. 
According to Proposition~\ref{pro:MC2} 
$E[T_1]\leq E[T_{1}^{*}]$ since $p_d'\leq p_d$, $p_r'\geq p_r$ and 
$p_c>p_m$. 
Finally, according to Proposition~\ref{pro:MC3},
$p_c'\leq p_c$ and $p_m'\leq p_m$ implies  $E[T_{1}^{*}]\leq E[T_{1}']$ 
and our proof is completed by using Lemma \ref{lem:mchain} to show that the last inequality of the statement holds.
\end{proof}

The algorithm may skip some levels or a new fitness level may be found before the whole population has reached the current fitness level.
Hence, by summing up the expected runtimes to leave each of the $n+1$ levels and the expected times for the whole population to takeover each level,
we obtain an upper bound on the expected runtime.
The next lemma establishes an upper bound on 
the expected time it takes to move from the absorbing state of the previous 
Markov chain ($S_{ 3,i}$) to any transient state ($S_{ 1,i+1}$ or $S_{ 2,i+1}$) 
of the next Markov chain. 
The lemma uses standard takeover arguments originally introduced in the first analysis of the ($\mu$+1)~EA for \om \cite{Witt2006}.
To achieve a tight upper bound Witt had to carefully wait for only a fraction of the population to take over a level before the next level was discovered.
In our case, the calculation of the transition probabilities of the MC is actually simplified if we wait for the whole population to take over each level. 
Hence in our analysis the takeover time calculations are more similar
 to the first analysis of the ($\mu$+1)~EA with and without diversity mechanisms 
to takeover the local optimum of
{\text{\sc TwoMax}\xspace} \cite{FriedrichOlivetoSudholtWittECJ2009}. 
%The second statement of the following lemma allows us to sum the expected takeover times for each level and the expected absorbing times to leave each level separately.
%
%
\begin{restatable}{lemma}{lemto}\label{lem-to}
 Let the best individual of the current population be in level $i$ and all 
individuals be in level at least $i-1$. 
Then, the expected time for the whole population to be in level at least $i$ is 
$\mathcal{O}(\mu \log{\mu})$. %The expected time to 
%take-over all the $n$ levels is $\mathcal{O}(n \mu^{2} \log{\mu})$.
\end{restatable}
\begin{proof}
  Let $k$ be the number of individuals of the population at 
fitness level $i$.  
Assume that one these $k$ solutions is selected as a parent. If the other 
parent is also on level $i$ but has a different genotype, then the Hamming 
distance between the parents is equal to $2d$ for some $d \in \mathbb{N}$ 
and the number of $1$-bits in the outcome of the crossover operator is 
$i-d$ plus a binomially distributed random variable with parameters $2d$ and 
$1/2$. With probability at least $1/2$ this random variable is larger or equal 
to $d$ due to the symmetry of the binomial distribution. If the other parent is 
on level $i-1$ then the Hamming distance between parents is $2d+1$ while the 
number of $1$-bits in the outcome of the crossover operator is $i-d-1$ plus a 
binomially distributed random variable with parameters $2d+1$ and 
$1/2$. The probability that the first $2d$ trials to have an outcome larger 
than $d$ is $1/2$. On top of that, for the offspring to have $i$ $1$-bits 
it is necessary that the fitter parent is picked for the final bit 
position with probability $1/2$. Hence, if at least one $i$ 
level solution is picked as a parent then with
probability at least $1/4$, the outcome of the crossover operator has at least 
$i$ $1$-bits. If the following mutation does not flip 
any bits, then a new solution with 
$i$ or more $1$-bits is added to the population. The 
solution will be accepted by selection unless the population has already been 
taken over. 
The probability that a solution at level $i$  is picked 
as a parent is at least $2k/\mu$ 
and 
the probability that mutation does 
not flip any bits is $(1-c/n)^n \geq 1/(e^c+1)$. So the expected time between 
adding the $k$th and 
the ($k+1$)th $i$-level solution to the population is less than 
$2(e^c+1)(\mu/k)$. 
By summing over all $k\in \{1,\ldots,\mu-1\}$, we obtain the following upper 
bound for 
the whole population to take over level $i$: \begin{align*} \sum_{k=1}^{\mu-1} 
2(e^c+1)(\mu/k)  
&\leq 2(e^c+1) \mu 
\sum_{k=1}^{\mu-1}1/k \\ 
&\leq 2(e^c+1) \mu \cdot \bo{(\log{\mu})}= \mathcal{O}(\mu \log{\mu}). 
\end{align*}
%By multiplying $\mathcal{O}( \mu^2 \log{\mu})$ with the number of 
%fitness levels, 
%$n$, we obtain the second statement.		
\end{proof}

The lemma shows that, once a new fitness level is discovered for the first time, 
it takes at most $\bo(\mu \log{\mu})$ generations until the whole population 
consists of individuals from the newly discovered fitness level or higher. While 
the absorption time of the Markov chain might decrease with the population size, 
for too large population sizes, the upper bound on the expected total take over 
time will dominate the runtime. As a result the MC framework will deliver larger 
upper bounds on the runtime unless the expected time  until the 
population takes over the fitness levels is asymptotically smaller than the 
expected absorption time of all MCs. For this reason, our results will require 
population sizes of $\mu=o(\log{n}/\log{\log{n}})$, to allow all fitness 
levels to be taken over in expected $o(n \log{n})$ time such that the latter 
time does not affect the leading constant of the total expected runtime.
% and not affect the leading constant in the runtime. 
\section{Upper Bound}
\label{sec:upperbound}
\makeatletter{}
In this section we use the Markov Chain framework devised in Section 
\ref{sec:framework} to prove that the \muGA is faster than any 
standard bit mutation-only 
($\mu+\lambda$)~EA.

In order to satisfy the requirements of Theorem~\ref{lem:mcdom}, we first show in 
Lemma~\ref{lem:easy} that $p_c>p_m$ if the population is at one of the final 
$n/(4c(1+e^{c}))$ fitness levels. 
The lemma also shows that it is easy for the algorithm to reach such a fitness 
level. 
Afterwards we bound the 
transition probabilities of the MC in Lemma \ref{lem:prob}. We conclude the section by 
stating and proving the main result, essentially by applying Theorem~\ref{lem:mcdom} with the transition probabilities calculated in Lemma \ref{lem:prob}. 

\begin{restatable}{lemma}{lemeasy}\label{lem:easy}

For the \muGA with mutation rate $c/n$ for any constant $c$, if the population is in any fitness 
level 
$i>n-n/(4c(1+e^{c}))$, then $p_c$ is always larger than $p_m$. The expected time 
for the \muGA to sample a solution in fitness level 
$n-n/(4c(1+e^{c}))$ for the first time is $\mathcal{O}(n\mu \log{\mu} )$.

\end{restatable}

\begin{proof}
We consider the probability $p_c$. If two individuals on the same fitness level 
with non-zero Hamming distance $2d$ are selected as parents with probability  
$p'$, then the probability that the crossover operator yields an improved 
solution is at least (see proof of Theorem 4 in \cite{Sudholt2015}):
\begin{align}
\Pr(X>d)=&\frac{1}{2} \left(1-\Pr(X=d)\right)= 
\frac{1}{2}\left(1-2^{-2d} {2d 
\choose d}\right) \geq 1/4, \label{eq:xo}
\end{align}
where $X$ is a binomial random variable with parameters $2d$ and $1/2$ which 
represents the number of bit positions where the parents differ and which are 
set to $1$ in the offspring. With probability 
$(1-c/n)^n$ no bits are flipped and 
the absorbing state is reached. If any individual is selected twice as parent, 
then the improvement can only be achieved  by  mutation (i.e., with 
probability $p_m$) since crossover is ineffective. So $p_c>p'  
(1/4)(1-c/n)^n+(1-p') p_m$, hence if $p_m< p'  (1/4)(1-c/n)^n+(1-p') p_m$ it 
follows that $p_m < p_c$. The condition can be simplified to $p_m < 
(1/4)(1-c/n)^n$ with simple algebraic manipulation. For large enough 
$n$, $(1-c/n)^n \geq 1/(1+e^{c})$ and the condition reduces to $p_m <  
1/(4(1+e^{c})) $. 

Since $p_m<(n-i) c/n$ is an upper bound on the transition probability (i.e., at 
least one of the zero bits has to flip to increase the $\om$ value), the 
condition is satisfied for  $i \geq n-n/(4c(1+e^{c}))$. For any level $i\leq 
n-n/(4c(1+e^{c}))$, after the take over of the level occurs in $\bo(\mu 
\log{\mu})$ expected time, the probability of 
improving is at least $\Omega(1)$ due to the linear number of $0$-bits that can 
be flipped. Hence, we can upper bound  the total number of generations 
necessary 
to reach fitness level $i=n-n/(4c(1+e^{c}))$  by $\mathcal{O}(n\mu \log{\mu})$.

\end{proof}

The lemma has shown that $p_c > p_m$ holds after a linear number of fitness 
levels have been traversed. Now, we bound the transition probabilities of the 
Markov chain.

\begin{restatable}{lemma}{lemprob}\label{lem:prob}
Let $\mu\geq3$. Then the transition probabilities $p_d$, $p_c$, $p_r$ and $p_m$ are bounded as 
follows:

\begin{align*}
p_d &\geq \frac{\mu}{(\mu+1)} \frac{i (n-i)c^2  }{n^2 (e^{c}+\bo(1/n)) }, && 
p_c 
\geq  \frac{\mu-1}{2 \mu^2 (e^{c}+\bo(1/n))},\\
p_r &\leq  \frac{(\mu-1) \big(2 \mu -1 +\bo(1/n)\big)}{2 e^{c} \mu^2 
(\mu+1) }, && p_m \geq \frac{c (n-i)}{n (e^c+\bo(1/n)) }.
\end{align*}

\end{restatable}
\begin{proof}

We first bound the probability $p_d$ of transitioning from the state $S_{1,i}$ 
to the state  $S_{2,i}$. In order to introduce a new solution at level  $i$  
with different genotype, it is sufficient that the mutation operator 
simultaneously flips  one of the $n-i$ 0-bits and one of the $i$  1-bits while 
not flipping any other bit.  We point out that in $S_{1,i}$, all individuals 
are identical, hence crossover is ineffective. Moreover, when the diverse 
solution is created, it should stay in the population, which occurs with 
probability $\mu/(\mu+1)$ since one of the $\mu$ copies of the majority 
individual should be removed by selection instead of the offspring.  So $p_d$ 
can be lower bounded as follows: 

\begin{align*}
p_d &\geq  \frac{\mu}{(\mu+1) } \frac{ic}{n} \frac{(n-i)c}{n} 
\left(1-\frac{c}{n}\right)^{n-2}.\\ 
\end{align*}

Using the inequality $(1-1/x)^{x-1} \geq 1/e \geq (1-1/x)^{x}$ , we now bound 
$\left(1-\frac{c}{n}\right)^{n-2}$ as follows:

\begin{align*}
\left(1-\frac{c}{n}\right)^{n-2}&\geq 
\left(1-\frac{c}{n}\right)^{n-1} \geq \left(1-\frac{c}{n}\right)^{n} \nonumber\\
&= \bigg(\left(1-\frac{c}{n}\right)^{(n/c)-1} 
\left(1-\frac{c}{n}\right)\bigg)^{c}\nonumber\\
&\geq \bigg(\frac{1}{e} 
\left(1-\frac{c}{n}\right)\bigg)^{c}\nonumber \geq \frac{1}{e^c} 
\left(1-\frac{c^2}{n}\right),\nonumber\\
\end{align*}
where in the last step we used the Bernoulli's inequality.
% \begin{align*}
%  1-\left(1-\frac{c}{n}\right)^c \leq c \cdot \frac{c}{n} 
% \implies 1-\frac{c^2}{n} \leq \left(1-\frac{c}{n}\right)^c
% \end{align*}

We can further absorb the $c^2/n$ in an asymptotic $\bo(1/n)$ term as follows:

\begin{align} 
\left(1-\frac{c}{n}\right)^{n-2} &\geq \left(1-\frac{c}{n}\right)^{n-1}\geq 
\left(1-\frac{c}{n}\right)^{n} \geq
\frac{1}{e^c}\left(1-\frac{c^2}{n}\right) \nonumber \\ 
&\geq e^{-c}-\bo(1/n)=\frac{1}{e^{c}+\bo(1/n)}.  \label{eq:series}
\end{align}
The bound for $p_d$ is then,
\begin{align*}
p_d&\geq 
\frac{\mu}{(\mu+1)} \frac{i (n-i)c^2 }{n^2 \big(e^{c}+\bo(1/n)\big)}.  
 \end{align*}

We now consider $p_c$. To transition from state $S_{2,i}$ to $S_{3,i}$ (i.e., 
$p_c$) it is sufficient 
that two genotypically different individuals are selected as parents (i.e., 
with 
probability at least $2(\mu-1)/\mu^2$), that crossover provides a better 
solution (i.e., 
with probability at least $1/4$ according to Eq.~\eqref{eq:xo}) and that 
mutation 
does not flip any 
bits (i.e., probability $(1-c/n)^n \geq 1/\big(e^{c}+ \bo(1/n)\big)$ according 
to 
Eq.~\eqref{eq:series}). Therefore, the probability is 

\begin{align*}
p_c &\geq 2\frac{\mu-1}{\mu^2 } \frac{1}{4} \left(1-\frac{c}{n}\right)^n\geq 
 \frac{\mu-1}{2 \mu^2(e^{c}+\bo(1/n))} 
\end{align*}

For calculating $p_r$ we pessimistically  assume that the Hamming distance  
between the individuals in the population is 2 and that there is always only 
one 
individual with a different genotype. A population in state $S_{2,i}$ which has 
diversity, goes back to state $S_{1,i}$ when: 

\begin{enumerate}
 \item A majority 
individual is selected twice as parent (i.e., probability 
$(\mu-1)^2/\mu^2)$, mutation does not flip any bit (i.e., probability $(1 - 
c/n)^n$) and the minority individual is discarded (i.e., probability 
$1/(\mu+1)$).
\item Two different individuals are selected as parents, crossover chooses 
either from the majority individual in both bit locations where they differ 
(i,e., prob. 1/4) and mutation does
not flip any bit (i.e., probability $(1 - c/n)^n\leq 1/e^c$)
or mutation must flip at least one specific bit (i.e., probability 
$\mathcal{O}(1/n)$). Finally, the minority individual is discarded (i.e., 
probability $1/(\mu+1)$).
\item A minority individual is chosen twice as parent and the mutation operator 
flips at least two specific bit positions (i.e., with probability $\bo(1/n^2)$) 
and finally the minority individual is discarded (i.e., probability 
$1/(\mu+1)$).
\end{enumerate}

Hence, the probability of losing diversity is:
\begin{align*}
p_r &\leq \frac{1}{\mu+1} \bigg[\frac{(\mu-1)^2}{\mu^2 } 
\left(1-\frac{c}{n}\right)^n + 2 
\frac{1}{\mu} \frac{\mu-1}{\mu } \left( 
\frac{1}{4}\left(1-\frac{c}{n}\right)^n + 
\mathcal{O}(1/n)\right) +\bo(1/n^2)\bigg]\\
 & \leq \frac{2(\mu-1)^2 + (\mu-1)+4 e^c (\mu-1)\bo(1/n)}{2 e^c \mu^2 (\mu+1)}+ 
\frac{  
\bo(1/n^2)}{ \mu+1}\\
&= \frac{(\mu-1)[2(\mu-1) + 1 + 4 e^c \bo(1/n)]+2 e^c \mu^2 \bo(1/n^2)}{2 e^c 
\mu^2 (\mu+1)}\\
&= \frac{(\mu-1)[2(\mu-1) + 1 + \bo(1/n)]+ \bo(1/n^2)}{2 e^c 
\mu^2 (\mu+1)}\\
& \leq \frac{(\mu-1) (2 \mu -1 +\mathcal{O}(1/n))}{2 e^{c} \mu^2 
(\mu+1) }. 
\end{align*}
In the last inequality we absorbed the $\bo(1/n^2)$ term into the 
$\bo(1/n)$ term. 

The transition probability $p_m$  from state $S_{1,i}$ to  
state $S_{3,i}$  is  the probability of improvement by mutation only, because 
crossover is ineffective at state $S_{1,i}$. The number of 1-bits in the 
offspring increases if the mutation operator flips one of the $(n-i)$ 0-bits ( 
i.e., with probability $c(n-i)/n$)  and does not flip any other bit (i.e., with 
probability $(1-c/n)^{n-1}\geq \big(e^c+\bo(1/n)\big)^{-1}$ according to 
Eq.~\eqref{eq:series}).
Therefore, the lower bound on the probability $p_m$ is:
\begin{align*}
p_m \geq  \frac{c(n-i)}{n \big(e^c+\bo(1/n)\big)}.
\end{align*}
\end{proof}

We are finally ready to state our main result.

\begin{theorem}\label{thm-main}

The expected runtime  of the \muGA with  $\mu \geq 3$ and 
mutation rate $c/n$ for any constant $c$  on
$\om$ is: 

\[E[T]\leq \frac{3e^{c} n \log{n}}{c(3+c)}   +  \bo{(n\mu \log{\mu} )}.\]

For $\mu=o(\log{n}/\log{\log{n}})$, the bound reduces to:

\[E[T]\leq \frac{3}{c (3+c)}   e^{c} n \log{n}\left(1+o(1)\right).\] 

\end{theorem}
\begin{proof}
We use  Theorem~\ref{lem:mcdom} to bound $E[T_i]$, the expected time until the 
\muGA creates an offspring at fitness level $i+1$ or above given that all 
individuals in its initial population are at level $i$. The bounds on the 
transition probabilities established in Lemma~\ref{lem:prob} will be set as the 
exact transition probabilities of another Markov chain, $M'$, with absorbing 
time  larger than $E[T_i]$ (by Theorem~\ref{lem:mcdom}). Since 
Theorem~\ref{lem:mcdom} requires that $p_c > p_m$ and Lemma~\ref{lem:easy} 
establishes that $p_c > p_m$ holds for all fitness levels $i>n-n/4c(1+e^{c})$, we 
will only analyse $E[T_i]$ for $n-1\geq i >n-n/\big(4c(1+e^{c})\big)$. Recall 
that, by Lemma~\ref{lem:easy}, level $n-n/\big(4c(1+e^{c})\big)$ is reached in 
expected $\bo(n \mu\log{\mu})$ time.

Consider the expected absorbing time $E[T_{i}']$, of the Markov chain $M'$ with 
transition probabilities:

\begin{align*}
p_d':= \frac{\mu}{(\mu+1)} \frac{i (n-i)c^2 }{n^2 (e^{c}+\bo(1/n))}, && 
p_c':=  \frac{\mu-1}{2 \mu^2 (e^{c}+\bo(1/n))},\\
p_r' := \frac{(\mu-1) \big(2 \mu -1 +\bo(1/n)\big)}{2 e^{c} \mu^2 
(\mu+1) }, && 
p_m':=\frac{c (n-i)}{ n (e^{c}+\bo(1/n)) }.
\end{align*}

According to Theorem~\ref{lem:mcdom}: % and Lemma~\ref{lem:mchain}:
\begin{equation}
\label{eq:ub}
E[T_{i}] \leq E[T_{i,1}']\leq\frac{p_c' + p_r'}{ p_c' p_d' + p_c' p_m' + p_m' 
p_r'} + \frac{1}{p_c'}.\\
\end{equation}

We simplify the numerator and the denominator of the 
first term separately. The numerator is 
\begin{align} \nonumber
 p_c'+p_r' &= \frac{\mu -1}{2 \mu ^2    (e^{c}+\bo(1/n)) }+\frac{(\mu 
-1) \big(2 \mu -1 +\bo(1/n)\big)}{2 e^{c} \mu ^2 (\mu +1)}\\
&\leq \frac{\mu -1}{2 \mu 
^2 
e^{c} }\left(1+\frac{2 \mu -1 +\bo(1/n)}{\mu 
+1}\right) 
 \leq \frac{(\mu -1) [3 \mu + \bo(1/n)] }{2 \mu^2 e^{c}   (\mu +1)}. \label{eq:nom}
\end{align}
We can also rearrange the denominator $D=p_c' p_d' + p_c' p_m' + p_m' p_r' $ as 
follows:
\begin{align}\nonumber
  D&= p_c'(p_d' + p_m') + p_m' p_r'\\ \nonumber
&=\frac{(\mu-1) \left(\frac{ \mu i (n-i)c ^2}{(\mu+1) n^2 \left(e^{c 
}+\bo(1/n)\right)}+\frac{c(n-i)  }{n \left(e^{c 
}+\bo(1/n) \right)}\right)}{2  \mu^2   \left(e^{c }+\bo(1/n)\right)}\\ 
\nonumber
&\text{   }+\frac{c (n-i) (\mu-1)    \big(2 \mu -1 +\bo(1/n)\big)}{ n 
\left(e^{c 
}+\bo(1/n)\right) 2 e^{c } \mu^2 (\mu+1) }\\ \nonumber
&\geq \frac{(\mu-1) \left(\frac{ \mu i (n-i)c ^2 }{(\mu+1) n^2}+\frac{c (n-i) 
}{n}\right)}{2 \mu^2    \left(e^{2 c }+\bo(1/n)\right)} +\frac{c (n-i) (\mu-1) \big(2 \mu -1 +\bo(1/n)\big)}
      { n \left(e^{2 c}+\bo(1/n)\right)2 \mu^2 (\mu+1)}\\ \nonumber
&\geq \frac{c (n-i) (\mu-1)  }{2 \mu^2 
\left(e^{2 c }+\bo(1/n)\right)} \cdot \left(\frac{ \mu i  c}{(\mu+1) n^2}+\frac{ 1
}{n}+\frac{2 \mu -1 +\bo(1/n)}{n (\mu+1)}\right)\\ \nonumber
&\geq \frac{c (n-i) (\mu-1)  }{2 \mu^2 
\left(e^{2 c }+\bo(1/n)\right)} \cdot \left(\frac{ \mu i  c +(\mu+1) n +n \big( 2 \mu -1 +\bo(1/n) \big)}{(\mu+1) 
n^2}\right)\\ \nonumber
%&\geq \frac{c (n-i) (\mu-1)  }{2 \mu^2 
%\left(e^{2 c }+\bo(1/n)\right)} \cdot \\\nonumber
%&\left(\frac{ \mu i  c + (\mu+1) n }{(\mu+1) n^2 }+\frac{n (2 
%\mu+\bo(1/n)-1)}{(\mu+1) n^2 }\right)\\ \nonumber
%
%
&\geq \frac{c (n-i) (\mu-1) \bigg(\mu i c + n \bigg[3\mu+\bo(1/n) \bigg] \bigg) }{2 \mu^2 \left(e^{2 c 
}+\bo(1/n)\right)(\mu+1) n^2   }. \\ 
\label{eq:dnm}
%
%
% &\geq \frac{c (n-i) (\mu-1)  (\mu i c +n \big[\mu+1+\beta (2 
% \mu+\bo(1/n)-1)\big])}{2   \mu^2 \beta \left(e^{2 c }+\bo(1/n)\right)(\mu+1) 
% n^2}\\ \nonumber
\end{align}
%
%\begin{equation*}
%\geq \frac{c (n-i) (\mu-1)  (\mu i c +n \big[\mu+1+\beta (2 
%\mu+\bo(1/n)-1)\big])}{2   \mu^2 \beta \left(e^{2 c }+\bo(1/n)\right)(\mu+1) 
%n^2}
%\end{equation*}
Note that the term in square brackets is the same in both the numerator (i.e., 
Eq.~\eqref{eq:nom}) and the 
denominator (i.e., Eq.~\eqref{eq:dnm}) including the small order terms in 
$\bo(1/n)$ (i.e., they are identical). Let $A=[3\mu + c'/n]$, where $c'>0$ is 
the 
smallest constant that satisfies the $\bo(1/n)$ in the upper bound on $p_r$ in 
Lemma~\ref{lem:prob}.
% Before putting the numerator and the denominator together we will bound the 
% expression:
% %\begin{equation}
%  %\frac{e^{2 
% %c }+\bo(1/n)}{e^{c }+\bo(1/n)}
% %\end{equation}
% %For a constant $c'$ we can bound the above expression as:
% %
% %
% \begin{align*}
%  \frac{e^{2 c }+\bo(1/n)}{e^{c }+\bo(1/n)} 
% &	\leq \frac{e^{2c }+c/n}
% 					{e^{c }} 
% 	\leq    \frac{e^{2c}}
% 							{e^{c }}        + \frac{c/n}
% 																				{e^{c }}                     \\
% &\leq e^{c} + \bo(1/n)
% \end{align*}
% where $c'$ is a constant.
We can now put the numerator  and denominator  together and simplify the 
expression : 
\begin{align*}
\frac{p_c'+p_r'}{p_c'(p_d' + p_m') + p_m' p_r'}&\leq \frac{ (\mu-1) A}{2 \mu^2 e^{c } (\mu+1) } \cdot
\frac{2  \mu^2 \left(e^{2 c }+\bo(1/n)\right)  (\mu+1)   n^2  }{c  (n-i) (\mu-1) (\mu i  c +n A)}\\
&\leq \frac{A  \left(e^{2 c }+\bo(1/n)\right)n^2}{e^{c 
} c (n-i)    (\mu i  c +n A)}.\\
\end{align*}

By using that $\frac{e^{2 c }+\bo(1/n)}{e^{c }}\leq e^{c} + 
\bo(1/n)$, we get:
\begin{align*}
\frac{p_c'+p_r'}{p_c'(p_d' + p_m') + p_m' p_r'}&\leq \frac{A  \left(e^{ c }+\bo(1/n)\right)n^2}{c (n-i)   ( \mu i  c +n A)} \\
&\leq e^{ c }\frac{A   n^2}{c (n-i)    (\mu i  c +n A)} + \bo(1/n) \frac{A  
n^2}{c (n-i)   (\mu i  c +n A)}. \\
\end{align*}

  The facts, $n-i \geq 1$, $A=\Omega(1)$, and $\mu, i, c>0$ imply that, $n A 
+\mu i c=\Omega(n)$ and  $\frac{A  n^2}{c (n-i)   ( \mu i  c +n A)}= \bo(n) $. 
When multiplied by the $\bo(1/n)$ term, we get:

\begin{align*}
&\leq \frac{e^{c} n}{c(n-i) }\frac{A n}{ (\mu i  c +n A)} + 
\bo(1).\\
\end{align*}

By adding and subtracting $\mu i c$ to the numerator of $\frac{A n}{ ( \mu i  c +n A)} $, we obtain: 

\begin{align*}
&\leq  \frac{e^{c} n}{c(n-i) } \left(1-\frac{\mu i c 
}{\mu i  c +n A }\right)+\bo(1).\\
\end{align*}

Note that the multiplier outside the brackets, $(e^{c} n )/ (c(n-i) )$, is in the order of $\bo\big(n/(n-i)\big)$. We now add and subtract $\mu n c$ to the numerator of $-\frac{\mu i c }{\mu i  c +n A}$  to create a positive additive term  in the order of $\bo\big(\mu(n-i)/n\big)$. 

\begin{align*}
&=\frac{e^{c} n}{c(n-i) }\left(1-\frac{\mu n c }{\mu i  c +n A }+\frac{\mu 
(n-i) c }{\mu i  c +n A }\right) +\bo(1)\\
&=\frac{e^{c} n}{c(n-i) }\left(1-\frac{\mu n c }{\mu i  c +n A }\right)+ \frac{e^{c} n}{c(n-i) }\frac{\mu 
(n-i) c }{\mu i  c +n A}+\bo(1)\\
&=\frac{e^{c} n}{c(n-i) }\left(1-\frac{\mu n c }{\mu i  c +n A 
}\right)+\bo(\mu).\\
\end{align*}
%
%
%Consider the term $\frac{\mu (n-i) c }{ n A+\mu i c }$. 
%Since  $A=\Omega(1)$,  and $\mu, c, i>0$, the denominator is 
%$\Omega(n)$. This allows us to  absorb the last term in brackets in a 
%$\mathcal{O}(\mu)$ term after we multiply it by $(e^{c} n )/ ((n-i) c)$. 
%%
%%
%%
%\begin{align*}
%&\leq  \frac{e^{c} n}{c(n-i) }\left(1-\frac{\mu n c }{n A+\mu i c }\right)+\bo(\mu)\\
%\end{align*}

Since $p_c'=\Omega(1/\mu)$, we can similarly absorb $1/p_c'$  into the 
$\mathcal{O}(\mu)$ term. 
 After the addition of the remaining term $1/p_c'$  from Eq.\eqref{eq:ub}, we 
obtain a valid upper bound on $E[T_{i}]$: 

\begin{align*}
E[T_{i}]&\leq\frac{p_c' + p_r'}{ p_c' p_d' + p_c' p_m' + p_m' p_r'}+ 
	    \frac{1}{p_c'}\leq\frac{e^{c} n }{c(n-i) } \bigg(1- \frac{\mu n c}{ \mu i  c +n A}  
  \bigg) + \mathcal{O}(\mu).
\end{align*}
In order to bound the negative term, we will rearrange its denominator (i.e., 
$n 
A+\mu i c$):
\begin{align*}
 n \big[3 \mu+ c'/n]+\mu i c&=3 \mu n + c' + \mu i c\\
&=3 \mu n + c' -(n-i) \mu c +\mu n c \\
&<  \mu n  (3+c)+ c',
%&=n\mu+ n + 2 n   \mu + n  \cdot \bo(1/n)-n  +\mu i c\\
%&=n\mu+ (1-1)n+2 n   \mu +\bo(1) -(n-i) \mu c +\mu n c \\
%%
\end{align*}
where the second equality is obtained by adding and subtracting $\mu n c$. 
Altogether,
\begin{align*}
E[T_{i}]&\leq\frac{e^{c} n}{c (n-i)} \bigg(1- 
\frac{\mu n c}{ \mu n (3 +c)+c'}  \bigg) + 
\mathcal{O}(\mu)\\
&=\frac{e^{c} n}{c (n-i)} \bigg(1-\frac{\mu n c+ c' \frac{c}{3+c}-c' 
\frac{c}{3+c}}{ \mu n (3 +c)+c'}  \bigg) + 
\mathcal{O}(\mu)\\
&= \frac{e^{c} n}{c (n-i)} \bigg(1- 
\frac{c}{ 3 +c} +\frac{ c' \frac{c}{3+c}}{\mu n (3 +c)+c'} \bigg)+ 
\mathcal{O}(\mu) \\
&= \frac{e^{c} n}{c (n-i)} \bigg(1- 
\frac{c}{ 3 +c} + \bo(1/n) \bigg)+ 
\mathcal{O}(\mu) \\
&= \frac{e^{c} n}{c (n-i)}   
\frac{3}{ 3+c}   + \mathcal{O}(\mu).
\end{align*}
%
%We can now replace $\beta$ with $n/(n-c)$ and obtain
%
%\begin{align*}
%E[T_{1,i}] &\leq\frac{e^{c} n}{c (n-i)} \frac{1 +2 \frac{n}{n-c}}{ 
%1+c+2 \frac{n}{n-c}}   + 
%\mathcal{O}(\mu)\\
%%
%%
%&\leq \frac					{e^{c} n}
										%{c (n-i)}   
															%\frac			{3 n-c} 
																	%{n (c +3)-c  (c +1)} + 	\bo(\mu) \\
%%
%%
%&\leq \frac					{e^{c} n}
										%{c (n-i)}   
															%\frac			{3 n} 
																	%{n (c +3)-c  (c +1)} + 	\bo(\mu) \\
%%
%%
%&\leq \frac					{e^{c} n}
										%{c (n-i)}   
															%\frac			{3 n -c(c+1)\frac{3}{3+c}+c(c+1)\frac{3}{3+c}}
																	%{n (c +3)-c  (c +1)} + 	\bo(\mu) \\
%%
%%
%&\leq \frac					{e^{c} n}
										%{c (n-i)}   
															%\frac			{3}
																	%{c+3}  
																	%+\frac			{c(c+1)\frac{3}{3+c}}
																	%{n (c +3)-c  (c +1)} + 	\bo(\mu) 
																	%\\
%%
%%
%&\leq  \frac{e^{c} n}{c (n-i)}  
%\frac{3}{c+3} + \bo(\mu)
%\end{align*}
%
%Where in the last inequality, we used the fact that $c=\bo(1)$, which implies that the second term in the second to last inequality is in the order of $\bo(1/n)$. 
If we add the expected time to take over each fitness level from 
Lemma~\ref{lem-to} and sum over all fitness levels the upper bound 
on the runtime is:
\begin{align*}
 &\sum\limits_{i=n-n/(4c(1+e^c))}^{n}\bigg(\frac{e^{c} n}{c (n-i)}  
\frac{3}{3+c} + \bo(\mu) + \bo{(\mu \log{\mu} )} \bigg)\\
&\leq \sum\limits_{i=0}^{n}\bigg(\frac{e^{c} n}{c (n-i)}  
\frac{3}{3+c} +  \bo{(\mu \log{\mu} )} \bigg) \\ 
&\leq\frac{3e^{c} n 
\log{n}}{c(3+c)}   +  \bo{(n\mu \log{\mu} )} \leq\frac{3e^{c} n \log{n}}{c(3+c)} 
  \left(1+ o(1) \right),
\end{align*}
where in the last inequality we use $\mu=o(\log{n}/\log{\log{n}})$ to prove the 
second statement of the theorem.
\end{proof} 

The second statement of the theorem provides an upper bound of $(3/4) e n \log 
n$ for the standard 
mutation rate $1/n$ (i.e., $c=1$) and $\mu=o(\log{n}/\log{\log{n}})$. 
% To find the mutation rate which yields the 
% smallest upper bound, we check when the first derivative of $ \frac{e^{c}}{c}  
% \frac{3}{c+3}$ with respect to $c$ is equal to zero.
% 
% \begin{align*}
%   \frac{d (\frac{e^{c} }{ c}  
% \frac{3}{c+3} )}{dc}&=-\frac{3 e^{c }}{c ^2 (c +3)}-\frac{3 
% e^{c }}{c  (c +3)^2}+\frac{3 e^{c }}{c  (c +3)}\\
% &=\frac{\left(3 e^{c }\right) \left(-\frac{1}{c +3}-\frac{1}{c 
% }+1\right)}{c  (c +3)}\\
% &=\frac{3 e^{c } \left(c ^2+c -3\right)}{c ^2 (c +3)^2}
% \end{align*}
The upper bound is minimised for 
$c=\frac{1}{2}\left(\sqrt{13}-1\right)$. Hence, the best upper bound is 
delivered for a mutation rate of about $1.3/n$. The resulting leading term 
of the upper bound is:

\begin{equation*}
 E[T]\leq \frac{6 e^{\frac{1}{2} 
\left(\sqrt{13}-1\right)}n \log{n}}{\left(\sqrt{13}-1\right) \left(\frac{1}{2} 
\left(\sqrt{13}-1\right)+3\right)}\approx 1.97 n \log {n}.
\end{equation*}

We point out that Theorem \ref{thm-main} holds for any $\mu \geq 3$.
Our framework provides a higher upper bound when $\mu=2$ compared to larger 
values of $\mu$.
The main difference  lies in the 
probability $p_r$ as shown in the following lemma. 

\begin{restatable}{lemma}{lemfortwo}
 The transition probabilities  $p_m$, $p_r$, $p_c$ and $p_d$ for the (2+1)~GA, with mutation rate $c/n$ and $c$ constant,
are bounded as follows:

\begin{align*}
p_d &\geq \frac{2}{3} \frac{i (n-i)c^2 }{n^2 (e^{c}+\bo(1/n))}, && 
p_c 
\geq  \frac{1}{8 (e^{c}+\bo(1/n))},\\
p_r &\leq  \frac{5}{24 e^{c}} + \bo(1/n),  && p_m \geq \frac{c 
(n-i)}{ (e^{c}+\bo(1/n)) n}.
\end{align*}
\end{restatable}
\begin{proof}
 While the other probabilities are obtained by setting $\mu=2$ in the 
expressions from Lemma~\ref{lem:prob}, the probability of losing diversity is 
larger for a population of size two than it is for $\mu\geq 3$. When either 
individual is picked twice as the parent (which occurs with probability $1/2$) 
and then the offspring is not mutated (which occurs with probability less than 
$1/e^c$), a copy of a solution is introduced into the population. Moreover, a 
copy can also be introduced if two different genotypes are selected and then 
crossover picks the same parents for the bit positions where the parents 
differ, 
which occurs with probability at most $1/4$. Any other event which produces a 
copy of one of the individuals requires flipping a constant number of specific 
bits which occurs with probability $\bo(1/n)$. Once a copy  is added to the 
population, the diversity is lost if the minority solution is removed from the 
population which occurs with probability $1/3$. Hence, by putting together the 
above probabilities we get

\begin{align*}
 p_r \leq& \frac{1}{3}\bigg(\frac{1}{2}e^{-c} + \frac{1}{2}\bigg(\frac{1}{4} 
e^{-c} + 
\bo(1/n)\bigg)\bigg)\\
   \leq&\frac{5}{24 e^{c}} + \bo(1/n).
\end{align*}

\end{proof}

% \begin{proof}
%  While the other probabilities are obtained by setting $\mu=2$ in the 
% expressions from Lemma~\ref{lem:prob}, the probability of losing diversity is 
% larger for a population of size two than it is for $\mu\geq 3$. When either 
% individual is picked twice as the parent (which occurs with probability $1/2$) 
% and then the offspring is not mutated (which occurs with probability less than 
% $1/e^c$), a copy of a solution is introduced into the population. Moreover, a 
% copy can also be introduced if two different genotypes are selected and then 
% crossover picks the same parents for the bit positions where the parents differ, 
% which occurs with probability at most $1/4$. Any other event which produces a 
% copy of one of the individuals requires flipping a constant number of specific 
% bits which occurs with probability $\bo(1/n)$. Once a copy  is added to the 
% population, the diversity is lost if the minority solution is removed from the 
% population which occurs with probability $1/3$. Hence, by putting together the 
% above probabilities we get
% 
% \begin{align*}
%  p_r \leq& \frac{1}{3}\bigg(\frac{1}{2}e^{-c} + \frac{1}{2}\bigg(\frac{1}{4} e^{-c} + 
% \bo(1/n)\bigg)\bigg)\\
%    \leq&\frac{5}{24 e^{c}} + \bo(1/n)
% \end{align*}
% 
% \end{proof}
%
The upper bound on $p_r$ from Lemma~\ref{lem:prob} is $1/(8e^{c})$, which is 
smaller than the bound we have just found. This is due to the assumptions in 
the lemma that there can be only one genotype in the population 
at a given time which can take over the population in the next iteration. 
However, when $\mu=2$, either individual can take over the population in the 
next iteration. This larger upper bound on $p_r$ for $\mu=2$ leads to a larger upper 
bound on the runtime of $E[T]\leq  \frac{4}{c+4}  \frac{e^{c} n \log{n}}{ c}   
(1+o(1))$ for the (2+1)~GA. The calculations are omitted as they are the same 
as those of the proof of Theorem \ref{thm-main}
where $p_r \geq 5/(24 e^{c}) + \bo(1/n)$ is used and $\mu$ is set to 2.
\section{Lower bound}
\label{sec:lowerbound}
\begin{algorithm2e}[t] \label{alg:grmu+1-GA}
    \caption{\suga 
    }

    $P \gets \mu \textrm{ individuals, uniformly at random from } \{0, 1\}^n$\;
    \Repeat{\textrm{termination condition satisfied}}
    {
            Choose $x, y \in P$ uniformly at random among $P^*$, the 
individuals with the current best fitness $f^*$\;
                 $z \gets$ Uniform crossover with probability $1/2$  $(x, y)$\;
            Flip each bit in $z$ with probability $c/n$\;  
            \textbf{If} $f(z)=f^*$ and $\max\limits_{w\in P^*}(HD(w,z))>2$
            \textbf{then} $z \gets z \vee \argmax\limits_{w\in 
P^*}(HD(w,z))$ \;
            
            $P \gets P \cup \{z\}$\;
             
        Choose one element from $P$ with lowest fitness and remove it from $P$, 
breaking ties at random; 
    }
  \end{algorithm2e}
% \begin{algorithm2e}[t] \label{alg:grmu+1-GA}
%     \caption{Greedy \muGA \cite{SarmaDeJongHANDBOOK,EIBENSMITH}
%     }
% 
%     $P \gets \mu \textrm{ individuals, uniformly at random from } \{0, 1\}^n$\;
%     \Repeat{\textrm{termination condition satisfied}}
%     {
%             Choose $x, y \in P$ uniformly at random among the individuals with 
% highest fitness\;
%                  $z \gets$ Uniform crossover with probability $1/2$  $(x, y)$\;
%             Flip each bit in $z$ with probability $c/n$\;  
%              $P \gets P \cup \{z\}$\;
%         Choose one element from $P$ with lowest fitness and remove it from $P$, 
% breaking ties at random; 
%     }
%   \end{algorithm2e}
In the previous section we provided a higher upper bound for the $(2+1)$ GA 
compared to the $(\mu+1)$ GA with population size greater than 2 and 
$\mu=o(\log{n}/\log{\log{n}})$. To 
rigorously prove that the $(2+1)$ GA is indeed slower, we require a lower 
bound on the runtime of the algorithm that is higher than the upper bound 
provided in the previous section for the $(\mu+1)$ GA ($\mu \geq 3$).

Since providing lower bounds on the runtime is a notoriously hard task, we will 
follow a strategy previously used by Sudholt \cite{Sudholt2015} and analyse a 
version of the \muGA with greedy parent selection and greedy crossover (i.e., Algorithm \ref{alg:grmu+1-GA}) in the sense that:
\begin{enumerate}
\item  Parents are 
selected uniformly at random only among the solutions from the highest fitness 
level ({\it greedy selection}).
\item If the offspring has the same fitness as its parents and its Hamming 
distance to any individual with equal fitness in the population is 
larger than 2, then the algorithm
automatically performs an OR 
operation between the offspring and the individual with the largest Hamming 
distance and fitness, breaking 
ties arbitrarily, and  adds the resulting offspring to the population i.e.,
we pessimistically allow it to skip as many fitness levels as possible ({\it 
semi-greedy crossover}).
\end{enumerate}
The greedy selection allows us to ignore the improvements that occur via 
crossover between solutions from different fitness levels. Thus, the crossover 
is only beneficial when there are at least two different genotypes in the 
population at the highest fitness level discovered so far. The difference with 
the algorithm analysed by Sudholt \cite{Sudholt2015} is that the \suga 
we consider does not use any diversity mechanism and it does not automatically 
crossover correctly when the Hamming distance between parents is exactly 2. As 
a result, there still is a 
non-zero probability of losing diversity before a successful crossover occurs. 
The crossover operator of the \suga is less greedy than the one analysed in 
\cite{Sudholt2015}
(i.e., there crossover is automatically successful also when the Hamming 
distance 
between the parents is 2). 
We point out that the upper bounds on the runtime derived in the previous 
section also hold for the greedy \suga.

The Markov chain structure of Figure \ref{fig-mchain} is still representative of 
the states that the algorithm can be in. When there is no diversity in the 
population, either an improvement via mutation occurs or diversity is introduced 
into the population by the mutation operator. When diversity is present, both 
crossover and mutation can reach a higher fitness level while there is also a 
probability that the population will lose diversity by replicating one of the 
existing genotypes. With a population size of two the diversity can be lost by 
creating a copy of either solution and removing the other one from the 
population during \emph{environmental selection} (i.e., Line 8 in 
Algorithm~\ref{alg:grmu+1-GA}). With population sizes greater than two, the 
loss of diversity can only occur when the majority genotype (i.e., the genotype 
with most copies in the population) is replicated. Building upon this we will 
show that the asymptotic performance  of \suga for \onemax  cannot be better 
than that of the  ($\mu$+1)~GAs for $\mu>2$. 

Like  in \cite{Sudholt2015} for 
our analysis we will apply the 
 fitness level method for lower bounds proposed by 
Sudholt \cite{SudholtTEVC2013}.

\begin{theorem} \cite{SudholtTEVC2013}
\label{thm:sudholt} 
Consider a partition of the search space into 
non-empty sets $A_1 ,\ldots,A_m$. For a search algorithm $\mathcal{A}$, we say 
that it is in $A_i$ or on level $i$ if the best individual created so 
far is in $A_i$ .  Let the probability of $\mathcal{A}$ traversing from level 
$i$ 
to level $j$ in one step be at most $u_i \cdot \gamma_{i,j}$ for all $j>i$ and 
$\sum_{j=i+1}^{m} \gamma_{i,j} = 1$ for all $i$. Assume that for all $j > i$ 
and 
some $0 
\leq \chi \leq 1$ it holds 

\begin{equation*}
\gamma_{i,j}\geq \chi \sum_{k=j}^{m} \gamma_{i,k}
\end{equation*}
Then the expected hitting time of $A_m$ is at least

\begin{align*}
&\sum_{i=1}^{m-1} \Pr(\mathcal{A}\text{ starts in } A_i) \cdot\bigg( 
\frac{1}{u_i} + \chi \sum\limits_{j=i+1}^{m-1} \frac{1}{u_j} \bigg)\\
\geq & 
\sum_{i=1}^{m-1} \Pr(\mathcal{A} \text{ starts in } A_i) \cdot \chi 
\sum\limits_{j=i}^{m-1} \frac{1}{u_j}.
\end{align*}

\end{theorem}

%
% \begin{theorem} \cite{SudholtTEVC2013}
% \label{thm:sudholt} 
% Consider a partition of the search space into 
% non-empty sets $A_1 ,\ldots,A_m$. For a search algorithm $\mathcal{A}$, we say 
% that it is in $A_i$ or on level $i$ if the best individual created so 
% far is in $A_i$ .  Let the probability of $\mathcal{A}$ traversing from level $i$ 
% to level $j$ in one step be at most $u_i \cdot \gamma_{i,j}$ for all $j>i$ and 
% $\sum_{j=i+1}^{m} \gamma_{i,j} = 1$ for all $i$. Assume that for all $j > i$ and some $0 
% \leq \chi \leq 1$ it holds 
% 
% \begin{equation*}
% \gamma_{i,j}\geq \chi \sum_{k=j}^{m} \gamma_{i,k}
% \end{equation*}
% Then the expected hitting time of $A_m$ is at least
% 
% \begin{align*}
% &\sum_{i=1}^{m-1} \Pr(\mathcal{A}\text{ starts in } A_i) \cdot\bigg( 
% \frac{1}{u_i} + \chi \sum\limits_{j=i+1}^{m-1} \frac{1}{u_j} \bigg)\\
% \geq & 
% \sum_{i=1}^{m-1} \Pr(\mathcal{A} \text{ starts in } A_i) \cdot \chi 
% \sum\limits_{j=i}^{m-1} \frac{1}{u_j}.
% \end{align*}
% 
% \end{theorem}
%
Due to the greedy crossover and the greedy 
parent selection used in \cite{Sudholt2015}, the 
population could be represented by the trajectory of a single individual. If an 
offspring with lower 
fitness was added to the population, then the greedy parent selection never 
chose it. If instead, a solution with equally high fitness and different 
genotype was created, then the algorithm immediately reduced 
the population to a single individual that is the best possible outcome 
from crossing over the two genotypes. %This 
%allows using the sum of the probability of creating a solution with equal 
%fitness value and $2k$ Hamming distance to its  parent  and the probability 
%of flipping $k$ more $1$-bits than $0$-bits as the probability of 
%improving $k$ fitness levels.
The main difference between the following analysis and that of  
\cite{Sudholt2015} is that 
 we want to take into account the possibility that the gained 
diversity 
may be lost before crossover exploits it.
To this end, when individuals of equal fitness and Hamming distance 2 are 
created, crossover only exploits this successfully (i.e., goes to the next 
fitness level) with the conditional probability that crossover is successful 
before the diversity is lost. Otherwise, the diversity is lost.
%
%diversity on the runtime, we cannot immediately perform the 
%best possible crossover. However, we instead use the conditional probability 
%that a successful crossover occurs before the diversity is lost, as a weight
%to 
%the probability of improving by creating $2k$ Hamming distance. Moreover, we 
%will restrict the application of conditional probability to $k=1$ since 
%those are the cases that affect the runtime asymptotically. 
Only when individuals of Hamming distance larger than 2 are created, we  
allow crossover to immediately provide the best possible outcome as in 
\cite{Sudholt2015}. 
Now, we can state the main result of this section.
\begin{restatable}{theorem}{thmlb}
\label{thm:lbound}
The expected runtime of the \suga with mutation probability $p=c/n$ for any 
constant $c$ on \onemax  is no less than: 
\[\frac{3e^{c}}{c\bigg(3+\max\limits_{1\leq 
k\leq n}(\frac{(n p)^k}{(k!)^2})\bigg)} n \log{n} - \mathcal{O}(n 
\log{\log{n}}).\]
\end{restatable}

\begin{proof}

To prove the theorem statement we wish to apply Theorem~\ref{thm:sudholt}. We 
say that the \suga is on level $i$ if its current best solution has $i$ 1-bits. 
It will suffice to calculate the runtime of the \suga starting from fitness 
level $\ell= \lceil n- \min \{n/\log{n}, n/(p^2\log{n})\}\rceil$. Given the 
greedy selection, the algorithm always reaches a new level in state $S_1$ 
(i.e., 
no diversity, see Fig.~\ref{fig-mchain}). We underestimate the 
expected runtime of the algorithm by considering as one single iteration the 
phases starting when state $S_2$ is reached and ending when state $S_2$ is 
left: 
either  a higher fitness level is reached (i.e., absorbing state) or the 
diversity is lost (i.e., back in $S_1$).

In order to apply  Theorem~\ref{thm:sudholt} we need to provide an upper bound 
on the probability $p_{i,j}$ of reaching fitness level $j$ from level $i$. In 
particular, we need to show that there exist  $u_i$ and $\gamma_{i,j}$ such 
that 
 $p_{i,j} \leq u_i \cdot \gamma_{i,j} $. We first concentrate on deriving a 
bound on $p_{i,j}$.

For the \suga with the current solution at level $i\geq \ell$, let $p_{i,i+k}$ 
be the 
probability that the algorithm reaches level $j=i+k$ in the next iteration. 

We first calculate the probability when $k \geq 2$. The \suga will be on level 
$i+k$ for $k\geq 2$ only if one of the following events occurs.

\begin{itemize}
  
  \item The mutation operator flips $k$ more 0-bits than 1-bits, which 
occurs with probability $p_{m,k}$; 
  \item The mutation operator flips exactly $k$ 
0-bits and $k$ 1-bits, which occur with probability $p_{d,k}$ (because the 
\suga 
automatically makes the largest possible improvement achieavable by crossover 
when Hamming distance greater than 2 is created); 
  \item The mutation operator 
flips exactly one 0-bit and one 1-bit which, leading to state $S_2$, initiates 
a 
phase that will end once $S_2$ is left. As said such a phase will be counted as 
one iteration only. To achieve this we calculate the conditional probability 
that level $i+k$ is reached before the diversity is lost (i.e., the algorithm 
has returned to state $S_1$). From $S_2$ level $i+k$ may be reached in either 
of 
the following ways: 
  
  \begin{itemize}
    \item After a successful crossover which increases the number of 
1-bits by one, the mutation operator flips $k-1$ more 0-bits than 
1-bits (which occurs with probability at most $p_{m,k-1}$);
    \item The mutation operator flips $k-1$ 0-bits and $k-1$ 1-bits, 
such that the resulting offspring has Hamming distance 2k with one of the 
existing solutions in the population which occurs with probability at most 
$p_{d,k-1}$.
  \end{itemize}

\end{itemize}

So the total probability, $p_{i, i+k}$ for $k\geq 2$ can be upper bounded as 
follows:
\begin{align}
p_{i, i+k} &\leq p_{m,k} + p_{d,k} + p_{d, 1} \frac{ p_{d, k-1} +  
p_{m, 
k-1}}{ p_{d, k-1} +p_{m, k-1}+p_r}\nonumber\\ 
&\leq p_{m,k} + p_{d,k} + p_{d, 1} \frac{ p_{d, k-1} + p_{m, k-1}}{p_r}. 
\label{eq1}
\end{align}

Here, the term multiplying $p_{d,1}$ is the conditional probability of reaching 
level $i+k$ before losing diversity, an event which occurs with probability 
$p_r$. In the conditional probability we use $1\cdot p_{m,k-1}$ instead of 
multiplying it by the crossover probability, which still gives a correct bound.

Overall, to bound the probability, $p_{i,i+k}$ we need upper bounds on 
the probabilities $p_{m,k}$, $p_{d,k}$, and a lower bound on $p_r$.

We start with $p_r$. To lose diversity it is sufficient that the outcome of 
crossover is a copy and then that mutation does not flip any bits (probability 
at least $(1-p)^n$). Finally we need environmental selection to remove the 
different individual (probability $1/3$). For the outcome of crossover to be a 
copy,  either the same individual is selected  twice (probability $1/2$) and 
crossover is ineffective (probability $1$)  or two different individuals are 
selected (probability $1/2$) and crossover picks both differing bits from the 
same parent (probability $1/2$). So, $p_r$, the probability of losing 
diversity, 
is at least,
\begin{align*}
p_r &\geq  \bigg(1-p\bigg)^n   \frac{1}{3} \bigg(\frac{1}{2} + \frac{1}{2} 
\cdot \frac{1}{2} 
 \bigg) = \frac{(1-p)^n}{4}. \\
%  &= \frac{(1-c/n)^n}{4}\\
%  &\geq \frac{e^{-c}(1-c/n) }{4}  =\frac{e^{-c} 
% }{4}-\mathcal{O}(1/n)  
\end{align*}
We derive $ p_{m,k}$  from Lemma 2 in \cite{SudholtTEVC2013} where it is proved 
that, for levels $i\geq \ell$, the probability that standard bit mutation with 
mutation rate $p$ flips $k$ more 0-bits than 1-bits is upper bounded by:
\begin{align}
p_{m,k} &\leq p^{k} (1-p)^{n-k} \frac{(n-i)^{k}}{k!} \cdot \bigg( 1+ 
\frac{3}{5} 
\cdot \frac{i(n-i)p^2}{(1-p)^2}\bigg) \nonumber \\
&= \frac{(1-p)^{n+k}}{k!} \bigg(\frac{p (n-i) }{(1-p)^2}\bigg)^{k} \bigg( 1+ 
\frac{3}{5} \cdot \frac{i(n-i)p^2}{(1-p)^2}\bigg) \nonumber\\
&\leq (1-p)^{n} \bigg(\frac{p (n-i) }{(1-p)^2}\bigg)^{k} \bigg( 1+ \frac{3}{5} 
\cdot \frac{i(n-i)p^2}{(1-p)^2}\bigg)\nonumber\\
&= (1-p)^{n} \bigg(\frac{p (n-i) }{(1-p)^2}\bigg)^{k} \bigg( 1+ 
\bo(1/\log{n}) \bigg). \label{eq:pmk}
\end{align}
Here, the last inequality follows from $i<n$ and $n-i\leq n-\ell$ which implies 
$i(n-i)p^2 = \bo(1/\log{n})$. Finally, the probability $p_{d,k}$ that the 
mutation operator flips exactly $k$ 0-bits and $k$ 1-bits is upper bounded as 
follows (see also Theorem 6 in \cite{Sudholt2015}):
\begin{align*}
p_{d,k} &\leq \frac{ (1-p)^n (n-i)^k p^{k} (n p)^{k}}{ k! k! (1-p)^{2 k}}\\
&\leq (1-p)^n \bigg(\frac{p (n-i) }{(1-p)^2}\bigg)^{k} \frac{(np)^{k}}{k! k!}.\\
\end{align*} 
Now, we separately bound some terms from Eq.~\eqref{eq1}:
\begin{align}
p_{d, 1}&\leq \frac{ (1-p)^n  (n-i) np^2 }{  (1-p)^{2} } \label{eq:dk1}\\
p_{d,k-1}+p_{m,k-1}&\leq (1-p)^n \bigg(\frac{p (n-i) 
}{(1-p)^2}\bigg)^{k-1}\cdot \bigg( \frac{(np)^{k-1}}{\big((k-1!)\big)^2 } + 1 + 
\bo(1/\log{n})\bigg)\nonumber \\
\frac{p_{d,1}}{p_r}&\leq \frac{ (1-p)^n  (n-i) np^2 }{  (1-p)^{2} } \cdot 
\frac{4}{(1-p)^n}\nonumber\\
&= 4 n p \frac{ p (n-i) }{  (1-p)^{2} }. \nonumber\\ \nonumber
\end{align}
\begin{align*}
 p_{d,1}\frac{ p_{d, k-1} + p_{m,k-1}}{p_r} &\leq 4 np (1-p)^n \bigg(\frac{p 
(n-i) 
}{(1-p)^2}\bigg)^{k} \cdot \bigg( \frac{(np)^{k-1}}{\big((k-1!)\big)^2} + 1 + \bo(1/\log{n})\bigg).\\
% \frac{p_{d, 1} p_{d, k-1}}{p_r}&\leq  \frac{\frac{ (1-p)^n  (n-i) np^2 }{  
% (1-p)^{2} }(1-p)^n \bigg(\frac{p (n-i) }{(1-p)^2}\bigg)^{k-1} 
% \frac{(np)^{k-1}}{(k-1!) (k-1)!}}{\frac{(1-p)^n}{4}}\\
% &\leq 4 (1-p)^n\bigg(\frac{p (n-i) }{(1-p)^2}\bigg)^{k} \frac{ 
% (np)^k}{\big((k-1)!\big)^2}  \\
% % 
% % \frac{4 p^k (n-i)^k (1-p)^{n} (n 
% % p)^k}{(1-p)^{2k}((k-1)!)^2}+\mathcal{O}(1/n)  \\
% % &\leq e^{-c}\bigg(\frac{p (n-i) }{(1-p)^2}\bigg)^{k} \frac{4 
% c^k}{(k-1)!}+\mathcal{O}(1/n)  \\
% % 
% % 
% \frac{p_{d, 1} p_{m, k-1}}{p_r}&\leq  \frac{\frac{ (1-p)^n  (n-i) np^2 }{  
% (1-p)^{2} }  (1-p)^{n} \bigg(\frac{p (n-i) }{(1-p)^2}\bigg)^{k-1} 
% }{\frac{(1-p)^n}{4}}\cdot \\
% &\bigg( 1+ \bo(1/\log{n})\bigg)\\
% % \frac{p_{d, 1} p_{m, k-1}}{p_r} &\leq\frac{4 n p^{k+1} (n-i)^k 
% % (1+\mathcal{O}(1/\log{n})) (1-p)^{n}}{(1-p)^{k+1}(k-1)!}\\
% % &+\mathcal{O}(1/n)  \\
% % 
% % &=4 (1-p)^n \bigg(\frac{p (n-i) }{(1-p)^2}\bigg)^{k} \frac{n p
% % }{(1-p)^{1-k}(k-1)!}\\
% % &\bigg( 1+ \frac{3}{5} \cdot \frac{i(n-i)p^2}{(1-p)^2}\bigg)\\
% \leq& 4 (1-p)^n \bigg(\frac{p (n-i) 
% }{(1-p)^2}\bigg)^{k} n p \big(1+ \bo(\log{n})\big)\\
\end{align*}

Therefore the upper bound on $p_{i,i+k}$ for $k\geq 2$ is:

\begin{align*}
p_{i,i+k}&\leq p_{m,k} + p_{d,k} + p_{d, 1} \frac{ p_{d, k-1} + p_{m,k-1}}{p_r} 
\\
&\leq (1-p)^n \left(\frac{p (n-i)}{(1-p)^2}\right)^{k}\cdot \\ &\text{ }\bigg[1+ \mathcal{O}(1/\log{n})+ \frac{(n p)^k}{(k!)^2}+4 n p \bigg(\frac{(n 
p)^{k-1}}{((k-1)!)^2}+1+\mathcal{O}(1/\log{n})\bigg) \bigg].\\
\end{align*}

For $y:= \max\limits_{1\leq k \leq n}(\frac{(n p)^k}{(k!)^2})$ the above bound 
reduces to:

\begin{align*}
p_{i,i+k}&\leq  (1-p)^n \left(\frac{p (n-i)}{(1-p)^2}\right)^{k}\cdot \bigg(1+\mathcal{O}(1/\log{n})+y+ 4np(y+1)\bigg),
\end{align*}
where we used that $np=c$, is a constant.

We now calculate the missing term of $p_{i,i+k}$ (i.e., $k=1$). For $k=1$, an 
improvement can be achieved only if one of the following events occurs:

\begin{itemize}
  \item The mutation operator flips one more 0-bit 
than it flips 1-bits which happens with probability  $p_{m,1}$ 
(Eq.~\eqref{eq:pmk} with $k=1$).
  \item The mutation operator flips exactly one 1-bit and one 0-bit with 
probability $p_{d,1}$ (Eq.~\eqref{eq:dk1}) and a phase in $S_2$ starts. Then, 
before the population loses its diversity either:
  \begin{itemize}
    \item A successful crossover between the 
two solutions with different genotypes occurs and the mutation does not 
flip more 1-bits than 0-bits (probability $p_{c}^*$),
    \item The crossover operator yields a solution on the same fitness level 
with probability $p_s$ and then the mutation operator flips one more 0-bit than 
it flips 1-bits with probability $p_{m,1}$.
    \item The crossover operator yields a solution on a worse fitness level 
(with probability less than $1-p_s$) and then the mutation operator flips two 
more 0-bits than it flips 1-bits which happens with probability  $p_{m,2}$ 
(Eq.~\eqref{eq:pmk} with $k=2$).
  \end{itemize}
\end{itemize}

So $p_{i,i+1}$ can be upper bounded as follows:

\begin{align}
p_{i,i+1}&\leq  p_{m,1} + p_{d,1} \frac{p_{c}^{*} + 
p_s p_{m,1}+ (1-p_s)p_{m,2}}{p_{c}^{*}+p_r+p_s p_{m,1}+ 
(1-p_s)p_{m,2}} \nonumber\\
%&\leq  p_{m,1} + p_{d,1} \frac{p_{c}^{*} + 
%p_{m,1}}{p_{c}^{*}+p_r+p_{m,1}}\nonumber\\
&\leq  p_{m,1} + p_{d,1} \frac{p_{c}^{*} + 
p_{m,1}}{p_{c}^{*}+p_r}.\nonumber\\
\end{align}
The second inequality is due to $p_s p_{m,1}+ (1-p_s)p_{m,2} \leq p_{m,1}$. 
Substituting the first $p_{m,1}$ term and $p_{d,1}$ we get:
\begin{align}
p_{i,i+1}&\leq (1-p)^n \frac{p(n-i)}{   
(1-p)^2}(1+\mathcal{O}(1/\log{n}))+\frac{(1-p)^n(n-i)np^2  }{  (1-p)^{2}} \cdot \frac{p_{c}^{*} + 
p_{m,1}}{p_{c}^{*}+p_r} \nonumber\\
%&= (1-p)^n \frac{p(n-i)}{   
%(1-p)^2}\bigg((1-p)(1+\mathcal{O}(1/\log{n}))\nonumber\\
%&+np  \frac{p_{c}^{*} + 
%p_{m,1}}{p_{c}^{*}+p_r}\bigg) \nonumber\\
&= (1-p)^n \frac{p(n-i)}{   
(1-p)^2}\bigg(1+\mathcal{O}(1/\log{n})+np  \frac{p_{c}^{*} + 
p_{m,1}}{p_{c}^{*}+p_r}\bigg). \label{eq2}
\end{align}

We now bound $p_{c}^{*}$. For crossover to increase the number of 1-bits by 
one, 
parent selection must pick two different individuals (probability $1/2$). Then 
1-bits have to be picked from the two positions where the parents differ 
(probability $1/4$). 
Finally, mutation must not flip more 1-bits than 0-bits. This event occurs
either  if  mutation does not flip any bits at all (with probability 
$(1-p)^n$) or if it flips at least one of the $n-i\leq \ell$ 0-bits (with 
probability at 
most $\bo(1/\log{n})$). So, the probability that 
crossover increases the number of 1-bits by one is

\begin{align*}
p_{c}^{*} &\leq  \frac{1}{2} \frac{1}{4} \bigg((1-p)^n+ 
\mathcal{O}(1/\log{n})\bigg)\\
& \leq \frac{(1-p)^n}{8} + \mathcal{O}(1/\log{n}).
% & \leq  \frac{e^{-c}}{8} + \mathcal{O}(1/\log{n})
\end{align*}
The probability $p_{m,1}$ is in the order of $\bo(1/\log{n})$ because the 
number 
of 0-bits is less than $\ell$. The term $(p_{c}^{*} + 
p_{m,1})/(p_{c}^{*}+p_r)$ in Eq.~\eqref{eq2} is therefore at most:
\begin{align*}
\frac{p_{c}^{*} + p_{m,1}}{p_{c}^{*}+p_r} &\leq  
\frac{\frac{(1-p)^n}{8}+\bo({1/\log{n})}}{\frac{(1-p)^n}{8}+\frac{(1-p)^n}{4}}\\
&=\frac{1}{3} + \bo(1/\log{n}).
\end{align*}

Hence, Eq. \eqref{eq2} is bounded as follows:

\begin{align*}
p_{i,i+1}&\leq (1-p)^n\frac{ p(n-i)  }{  (1-p)^{2}} 
\bigg(1+\bo(1/\log{n})+np\big( \frac{1}{3} + \bo(1/\log{n})\big) \bigg) \\ 
&\leq (1-p)^n \frac{ p(n-i) }{ (1-p)^{2}}  \bigg(1+\frac{c}{3} + 
\mathcal{O}(1/\log{n})\bigg) \\ 
&\leq (1-p)^n \frac{p(n-i)  }{  (1-p)^{2}}  \bigg(\frac{3+c}{3} + 
\mathcal{O}(1/\log{n})\bigg). 
\end{align*}

We now can determine the parameters $u_i$ and $\gamma_{i,k}$ for the 
application 
of 
Theorem~\ref{thm:sudholt}. We define,

\begin{align}
u_{i}^{'}&:=  \frac{e^{-c} (n-i)p  }{  (1-p)^{2}} \bigg(\frac{3+y}{3} + 
\mathcal{O}(1/\log{n})\bigg),\label{eq:ui}\\
\gamma^{'}_{i,i+k}&:= \left(\frac{(3+12c)p (n-i)}{ (1-p)^2}\right)^{k-1}. 
\label{eq:gi}
\end{align}

Observe that $y\geq c$ and for large enough $n$, 
\\$u_{i}^{'}\gamma^{'}_{i,i+k}\geq 
p_{i,i+k}$.
Consider the normalised variables \\$u_i:= u_{i}^{'} \sum_{j=i+1}^{n} 
\gamma_{i,j}^{'}$ and $\gamma_{i,j}:=\frac{\gamma_{i,j}^{'}}{\sum_{j=i+1}^{n} 
\gamma_{i,j}^{'}}$. ~Since \\$u_{i}\gamma_{i,j}=u_{i}^{'}\gamma^{'}_{i,j}\geq 
p_{i,j}$, it follows that $u_i$ and $\gamma_{i,j}$  satisfy their definitions in
Theorem~\ref{thm:sudholt}.  

Now we turn to the main condition of Theorem~\ref{thm:sudholt}.

\begin{align*}
\sum_{k=j-i}^{n-i} \gamma_{i,i+k}^{'} &\leq \bigg(\frac{(3+12c)p (n-i)}{ 
(1-p)^2}\bigg)^{j-i-1} \cdot \sum_{k=0}^{\infty} \bigg(\frac{(3+12c)p (n-i)}{ 
(1-p)^2}\bigg)^{k}.\\
\end{align*}
For large enough $n$ , $\frac{(3+12c)p (n-i)}{ (1-p)^2} < 1$. Therefore,
\begin{align*}
\sum_{k=j-i}^{n-i} \gamma_{i,i+k}^{'}&\leq \gamma_{i,j}^{'}  \frac{1}{1- \frac{(3+12c)p (n-i)}{ (1-p)^2}}\\
&\leq \gamma_{i,j}^{'}  \frac{1}{1- \mathcal{O}(1/\log{n})}\\
\end{align*}

Since $\gamma^{'}_{i,j}\geq \sum_{k=j}^{n} \gamma^{'}_{i,k}$ implies 
$\gamma_{i,j}\geq \sum_{k=j}^{n} \gamma_{i,k}$, for $\chi = 1- 
\mathcal{O}(1/\log{n})$ the main condition of Theorem~\ref{thm:sudholt} is 
satisfied. 

All that remains is to calculate $u_i$.
\begin{align*}
u_i&= u^{'}_i \sum_{j=i+1}^{n} \gamma_{i,j}^{'} \leq u^{'}_i \sum_{k=1}^{n-i} 
\left(\frac{(3+12c)p (n-i)}{ (1-p)^2}\right)^{k-1}\\
&\leq u^{'}_i \sum_{k=0}^{\infty} 
\left(\frac{(3+12c)p (n-i)}{ (1-p)^2}\right)^{k}\\
&\leq \frac{e^{-c} (n-i)p  }{  (1-p)^{2}} \bigg(\frac{3+y}{3} + 
\mathcal{O}(1/\log{n})\bigg)  \frac{1}{1- \mathcal{O}(1/\log{n})},\\
%&\leq \frac{e^{-c} (n-i)p  }{  (1-p)^{2}} \bigg(\frac{3+y}{3} + 
%\mathcal{O}(1/\log{n})\bigg)  (1 + \mathcal{O}(1/\log{n}))  \\
%&\leq \frac{e^{-c} (n-i)p  }{  (1-p)^{2}} \bigg(\frac{3+y}{3} + 
%\mathcal{O}(1/\log{n})\bigg) 
\end{align*}
where in the first inequality we substituted $\gamma_{i,j}^{'}$ from 
Eq.~\eqref{eq:gi} and in the last $u_{i}^{'}$ from Eq.~\eqref{eq:ui}.
Overall, according to Theorem~\ref{thm:sudholt} the runtime is
\begin{align*}
E\left[ T \right] &\geq Pr\{\text{Initialise at level $\ell$ or below}\} \cdot 
\chi \cdot \sum_{i=\ell}^{n-1}\frac{1}{u_i}\\
&\geq (1-1/\log{n}) (1- \mathcal{O}(1/\log{n})) \bigg(\frac{3+y}{3} + \mathcal{O}(1/\log{n})\bigg)^{-1} \cdot \\
& \frac{e^{c} 
(1-p)^2}{p} (1- \mathcal{O}(1/\log{n})) \sum_{i=\ell}^{n-1} \frac{1}{n-i}. \\
\end{align*}
Finally, we combine the $(1-\bo(1/\log{n})$ terms and 
$(1-p)^2=1-\bo(1/\log{n})$ into
a single $(1-\bo(1/\log{n})$ term and change the index of the sum:
\begin{align*}
&\geq  (1- \mathcal{O}(1/\log{n}))  \bigg(\frac{3}{3+y} - \mathcal{O}(1/\log{n})\bigg) \frac{e^{c}}{p} 
\sum_{i=1}^{n-\ell} 
\frac{1}{i}\\
&\geq (1- \mathcal{O}(1/\log{n}))   \bigg(\frac{3}{3+y} - 
\mathcal{O}(1/\log{n})\bigg) \frac{e^{c}}{p} \log{(n-\ell)} \\
&\geq (1- \mathcal{O}(1/\log{n})) \bigg(\frac{3}{3+y} - 
\mathcal{O}(1/\log{n})\bigg) \frac{e^{c}}{p} (\log{n}-\log{\log{n}}) \\
& \geq\bigg(\frac{3}{3+y} - \mathcal{O}(1/\log{n})\bigg) \frac{e^{c}}{p} 
(\log{n}-\bo(\log{\log{n}})) \\
&\geq\frac{3 e^{c}  n \log{n}}{c(3+\max\limits_{1\leq k \leq n}(\frac{(n 
p)^k}{(k!)^2}))} 
-\mathcal{O}(n \log{\log{n}}).
\end{align*}
\end{proof}

Note that for $c\leq 4$, $\max\limits_{1\leq k \leq n}(\frac{(n p)^k}{(k!)^2})\leq p 
n=c$. Since $E[T]\geq e n \log{n}$ for $c\geq 3$, for the purpose of finding the 
mutation rate that minimises the lower bound, we can reduce the statement of 
the theorem to:
\begin{equation*}
   \frac{3 e^{c}  n \log{n}}{c(3+c)} -\bo(n\log{\log{n}}).
\end{equation*}
The theorem provides a lower bound of $(3/4)e n \log{n} -\bo(n\log{\log{n}})$ 
for the 
standard mutation rate $1/n$ (i.e., c = 1).
The lower bound is minimised for $c=\frac{1}{2}\left(\sqrt{13}-1\right)$. Hence, the smallest lower bound is delivered for a mutation rate of about $1.3/n$. The resulting lower bound is :
\begin{align*}
 E[T]\geq& \frac{6 e^{\frac{1}{2} 
\left(\sqrt{13}-1\right)}n \log{n}}{\left(\sqrt{13}-1\right) \left(\frac{1}{2} 
\left(\sqrt{13}-1\right)+3\right)} -\bo(n\log{\log{n}}) \\
\approx& 1.97 n \log {n} -\bo(n\log{\log{n}}).
\end{align*}
Since the lower bound for the \suga matches the upper bound for the \muGA with 
$\mu>2$, the theorem proves that, under greedy selection and semi-greedy crossover, populations of size 2 
cannot be faster than larger population sizes up to 
$\mu=o(\log{n}/\log{\log{n}})$.
In the following section we give experimental evidence that the greedy 
algorithms 
are faster than the standard \GAtwo, thus suggesting that the same conclusions 
hold also for the standard non-greedy algorithms. %\footnote{We thank an anonymous reviewer for pointing out that it is not obvious that the greedy GA is 
%faster than the standard GA.}.
%
\section{Experiments}
\label{sec:experiments}

\begin{figure}[t]
\caption{Average runtime over 1000 independent runs versus problem size $n$.}
\label{fig:ave}
\includegraphics[width=\textwidth]{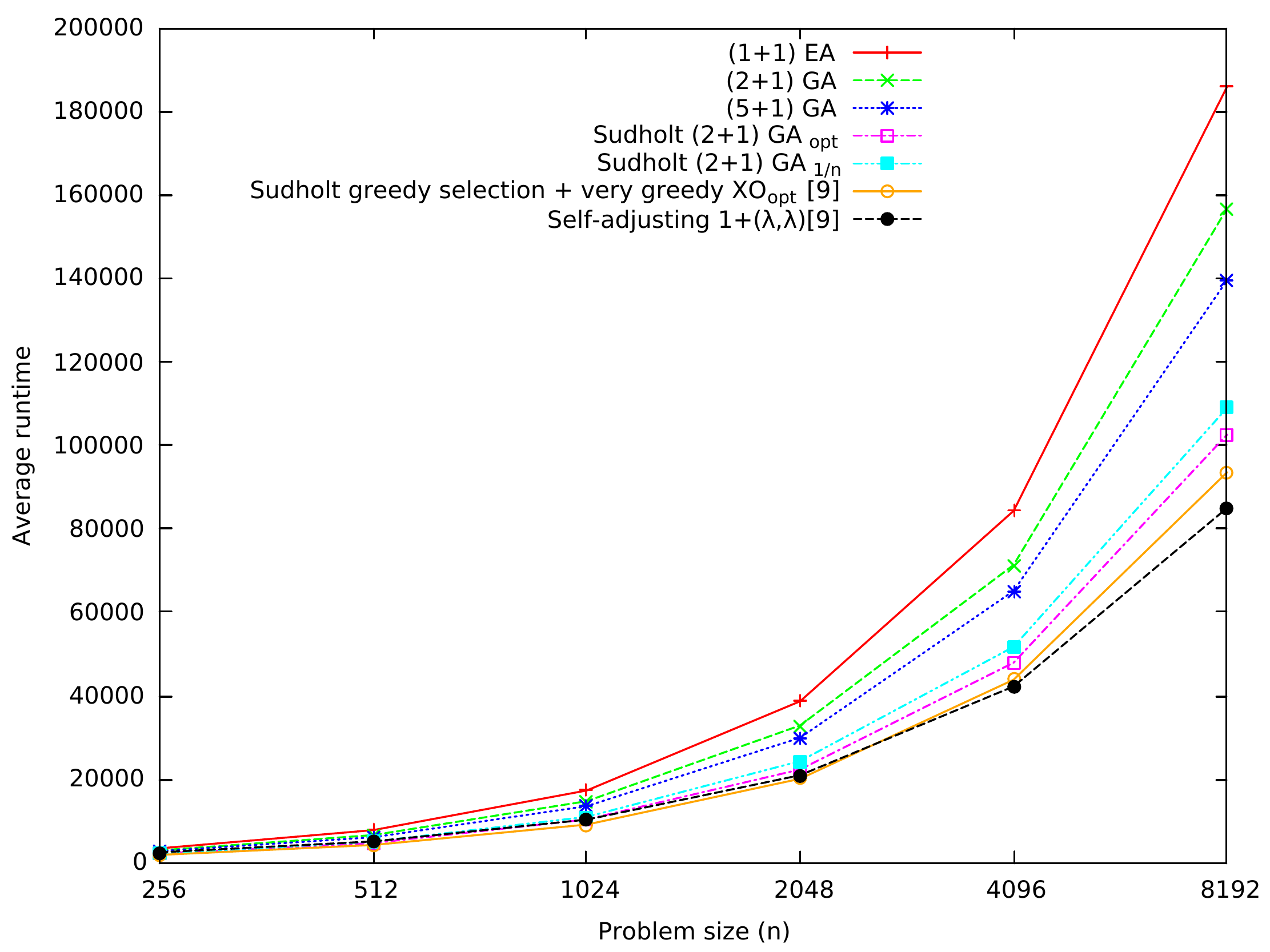}
\end{figure}
\begin{figure}[t]
\caption{Comparison between standard selection, greedy selection and greedy 
selection + greedy crossover GAs. The runtime is averaged over 1000 independent 
runs. }
\label{fig:greed}
\includegraphics[width=\textwidth]{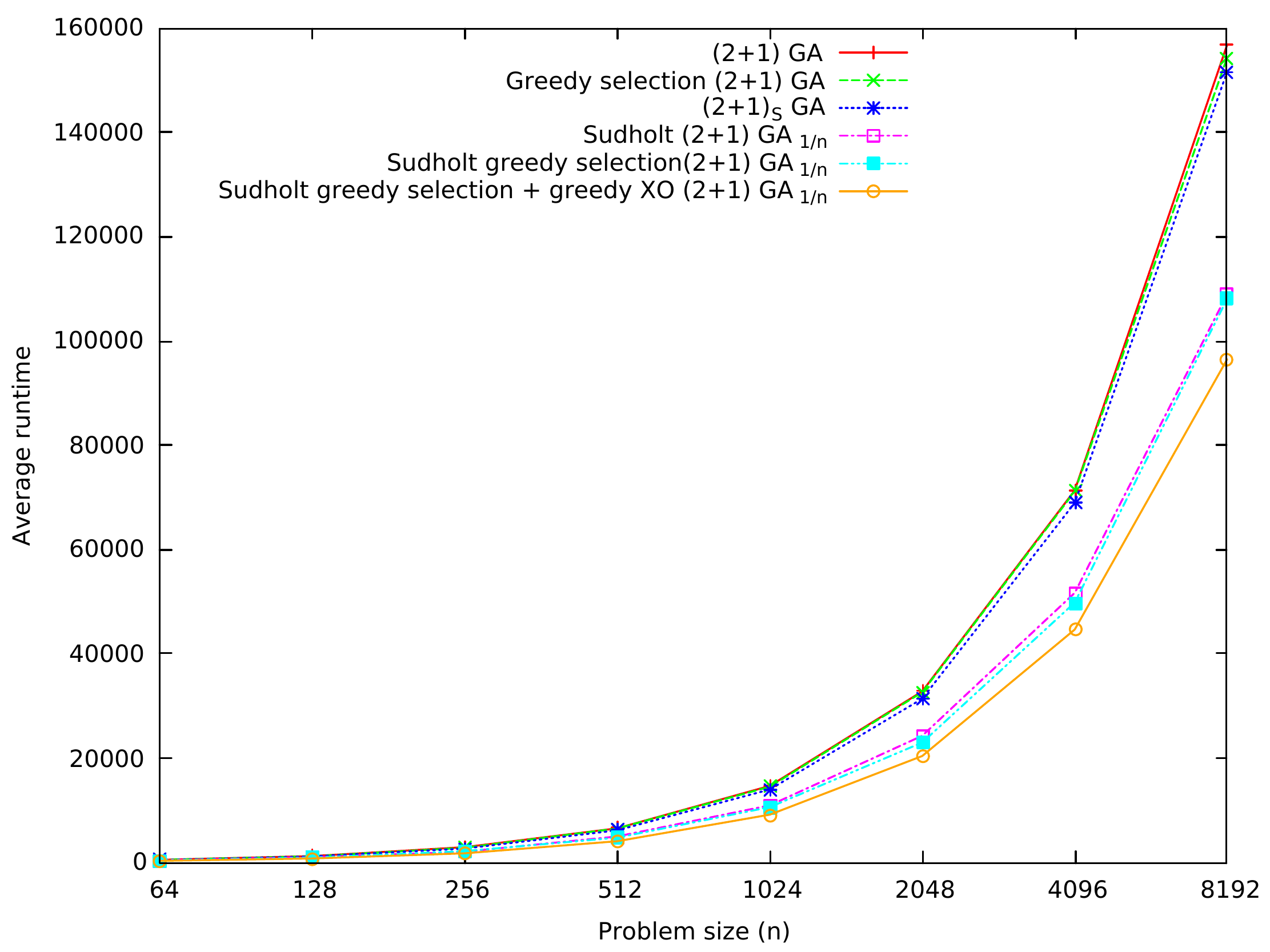}
\end{figure}
\begin{figure}[t]
\caption{Average runtime gain of the \muGA versus the \GAtwo for different 
population sizes, errorbars show the standard deviation normalised by the 
average runtime for $\mu=2$.}
\label{fig:2GA}
\includegraphics[width=0.9\textwidth]{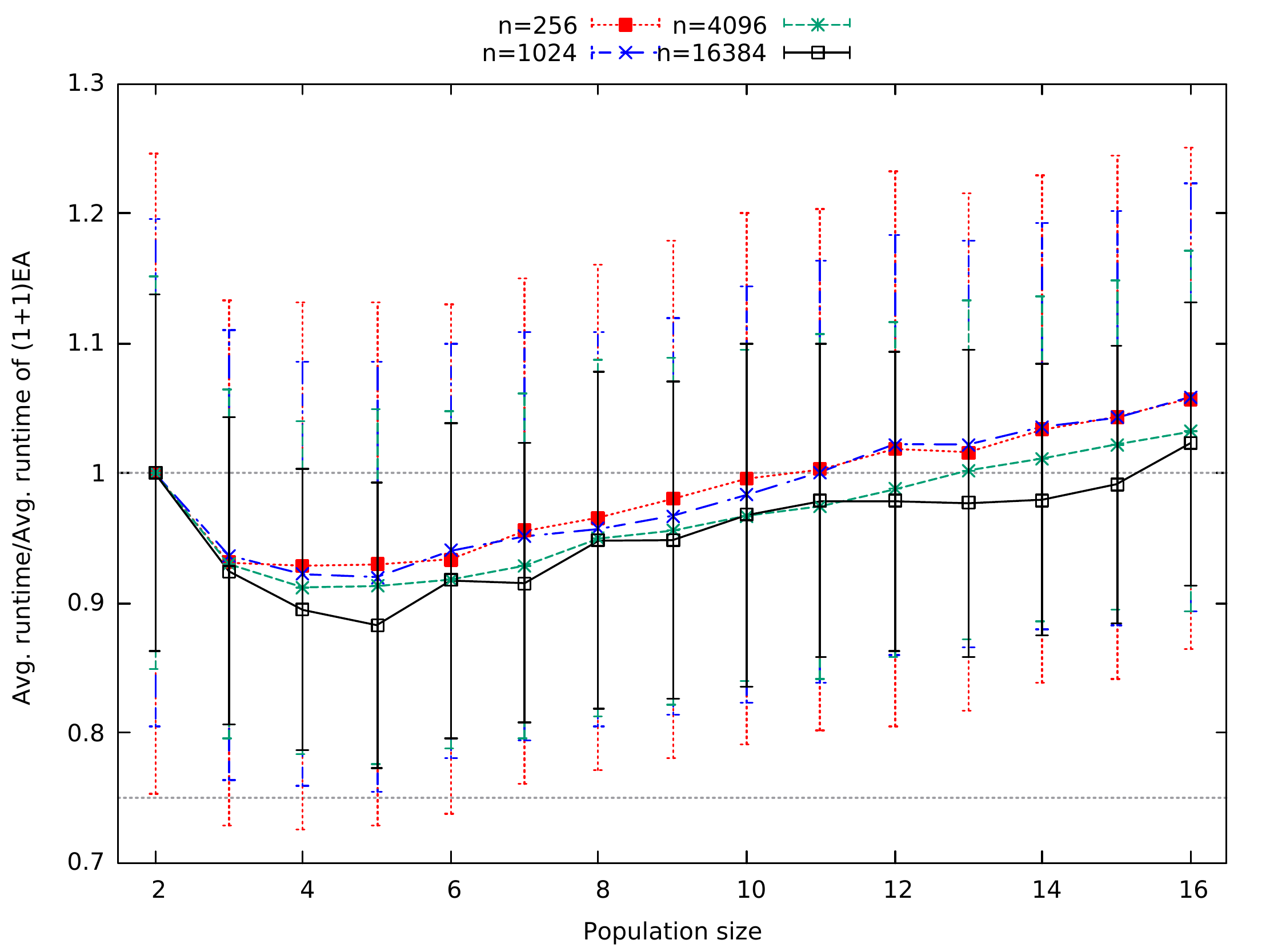}
 \end{figure}
The theoretical results presented in the previous sections pose some new 
interesting questions. On one hand, the theory suggests that population sizes 
greater than 2 benefit the \muGA for hillclimbing the \onemax function. On the 
other hand, the best runtime bounds are obtained for a mutation rate of approximately $1.3/n$, suggesting that higher mutation rates than the standard $1/n$ rate may 
improve the performance of the \muGA. In this section we present the outcome of 
some experimental investigations to shed further light on these questions. In 
particular, we will investigate the effects of the population size and 
mutation rate on the runtime of the steady-state GA for \om  and compare its 
runtime against other GAs that have been proved to be faster 
than mutation-only EAs in the literature. 

We start with an overview of the performance of the algorithms.
% we discussed in previous sections.
  In Fig. \ref{fig:ave},  we plot the average 
runtime over 1000 independent runs of the \muGA with $\mu=2$ and $\mu=5$ (with 
uniform parent 
selection and standard $1/n$ mutation rate) for exponentially increasing problem 
sizes and compare it against the fastest standard bit mutation-only EA with 
static mutation rate (i.e., the (1+1)~EA with $1/n$ mutation rate). 
While the algorithm using $\mu=5$ outperforms the $\mu=2$ version,
they are both faster than the (1+1)~EA already for small problem sizes. We also 
compare the algorithms against the (2+1)~GA investigated by Sudholt 
\cite{Sudholt2015} where diversity is enforced by the environmental selection always preferring distinct 
individuals of equal fitness - the same GA variant that was first proposed and 
analysed in  \cite{Jansen2002}. We run the 
algorithm both with standard mutation rate $1/n$ and with the optimal mutation 
rate $(1+\sqrt{5})/(2n)$. Obviously, when diversity is enforced, the 
algorithms are faster.
Finally, we also compare the algorithms against the (1+($\lambda$,$\lambda$))~GA with self-adjusting population sizes 
and Sudholt's (2+1)~GA as they were compared previously in \cite{DoerrDoerrEbel2015}. Note that in \cite{DoerrDoerrEbel2015} (Fig. 8 therein) 
Sudholt's algorithm was implemented with a very greedy parent selection operator that always prefers distinct individuals on the highest fitness level for reproduction.

In order to decompose the effects of the greedy parent selection, greedy crossover and the use of 
diversity, we conducted further experiments shown in Figure~\ref{fig:greed}. 
Here, we see that it is indeed the enforced diversity that creates the 
fundamental performance difference. Moreover, the results show that the greedy 
selection/greedy crossover GA is slightly faster than the greedy parent 
selection GA and that  
greedy parent selection is slightly faster than standard selection. Overall, the figure suggests that the lower 
bound presented in Theorem~\ref{thm:lbound} is also valid for the standard \GAtwo with 
uniform parent selection (i.e., no greediness). In Figure~\ref{fig:greed}, it 
can be noted that the performance difference between the GA with greedy 
crossover and greedy parent selection analysed 
in \cite{Sudholt2015} and the (2+1)~GA with enforced diversity and without 
greedy crossover is more pronounced than the performance difference between the 
standard (2+1)~GA analysed in Section~\ref{sec:upperbound} and the \suga which 
was analysed in Section~\ref{sec:lowerbound}. The reason behind the difference 
in discrepancies is that the \suga does not implement the greedy crossover 
operator 
when the Hamming distance is 2. We speculate that cases where the Hamming 
distance 
is \emph{just} enough for the crossover to exploit it occur much 
more frequently  than the cases where a larger Hamming distance is present. As 
a result, the performance of the \suga does not deviate much from the standard 
algorithm. Table~\ref{tbl:stat} presents the 
 mean and standard deviation of the runtimes of the algorithms 
depicted in Figure~\ref{fig:ave} and Figure~\ref{fig:greed} over 1000 
independent runs.

Now we investigate the effects of the population size on the \muGA. We 
perform 1000 independent runs of the \muGA with uniform parent selection and 
standard 
mutation rate $1/n$ for increasing population sizes up to $\mu=16$. In 
Fig.~\ref{fig:2GA} we present average runtimes divided by the runtime of the 
$(2+1)$~GA and in Fig.~\ref{fig:EA} normalised against the runtime of the 
(1+1)~EA. In both figures, we see that the runtime improves for $\mu$ larger 
than $2$ and after reaching its lowest value increases again with the population 
size. It is not clear whether there is a constant optimal static value for $\mu$ around 
4 or 5. The experiments, however, do not rule out the  possibility that the 
optimal static population size increases slowly with the problem size (i.e., $\mu=3$ 
for $n=256$, $\mu=4$ for $n=4096$ and $\mu=5$ for $n=16384$).  
On the other hand, 
we clearly see that as the problem size increases we get a larger improvement on 
the runtime. This indicates that the harder is the problem, more useful are the 
populations.  In particular, in  Figure~\ref{fig:EA} we see that the 
theoretical 
asymptotic gain of 25\% with respect to the runtime of the (1+1)~EA is  
approached more and more closely as $n$ increases. For the considered problem 
sizes, the \muGA is faster than the (1+1)~EA for all tested values of $\mu$. 
However, to see the runtime improvement  of the \muGA 
against the \GAtwo  for $\mu > 15$  the experiments (Fig.~\ref{fig:2GA}) suggest 
that greater problem sizes would need to be used.

Finally, we investigate the effect of the mutation rate on the runtime. Based on 
our previous experiments we set the population size to the best seen value of $\mu=5$ 
and perform 1000 independent runs for each $c$ value ranging from $0.9$ to 
$1.9$. In 
Figure~\ref{fig:chi}, we see that even though the mutation rate $c\approx 1.3$ 
minimises the upper bound we proved on the runtime, setting a 
larger mutation rate of $1.6$ further decreases the runtime. 

%\afterpage{

%\resizebox{.5\linewidth}{!}{
\newgeometry{margin=1cm}

\begin{landscape}\begin{table*}[t]
\small

 \caption{Statistics for the experimental results of Fig.~\ref{fig:ave} and 
Fig.~\ref{fig:greed}}
\label{tbl:stat}
\resizebox{\linewidth}{!}{
\begin{tabular}[]{l || c | c || c | c || c | c || c | c }
\multirow{2}{*}{\textbf{Algorithms}}& 
  \multicolumn{2}{c||}{$n=64$}&  
\multicolumn{2}{c||}{$n=128$} &  \multicolumn{2}{c||}{$n=256$} &  
\multicolumn{2}{c}{$n=512$}  \\ \cline{2-9}
& Mean & Std. dev.& Mean & Std. dev.& Mean & Std. dev.& 
Mean & Std. dev.\\  \hline
 (1+1)~EA &612.66 & 208.88 & 1456.81 & 450.51 & 3397.72 & 887.07 & 7804.65 & 
1791.44  \\ 
 (2+1)~GA & 546.57 & 179.61 & 1271.30 & 357.41 & 2952.70 & 727.84 & 6586.60 & 
1378.50   \\
Greedy (2+1)~GA & 519.93 & 177.28 & 1228.86 & 355.23 & 2854.18 & 730.07 & 
6548.51 & 1434.21 \\ 
 (5+1)~GA & 529.29 & 156.19 & 1194.50 & 281.92 & 2744.80 & 595.41 & 6087.60 & 
1164.30 \\ 
 Sudholt`s (2+1)~GA$_{1/n}$ &449.92 & 121.87 & 1040.26 & 277.71 & 2099.24 & 
375.74 & 5022.30 & 873.47 \\ 
 Sudholt`s (2+1)~GA$_{opt}$ &427.87 & 108.51 & 978.50 & 212.13 & 
2142.82 & 411.30 & 4682.88 & 742.95 \\ 
 (2+1)$_{S}$ GA&484.40 & 174.87 & 1183.80 & 366.28 & 2705.09 & 710.85 
& 6183.55 & 1451.07 \\ 
 Sudholt's greedy selection + greedy XO (2+1)~GA$_{1/n}$&326.21 & 108.04 & 
790.47 & 223.50 & 1787.49 
& 425.70 & 4105.20 & 851.14 \\ 
 Sudholt`s greedy selection (2+1)~GA$_{1/n}$&410.56 & 117.45 & 958.64 & 222.26 
& 
2192.50 & 
469.57 & 4801.66 & 901.88 \\ 
 Sudholt's (2+1)GA diverse crossover$_{opt}$ &386.14&104.59&907.719 
 &222.24&1873.37&409.68&4243.65&812.38 \\
 Self-adjusting $1+(\lambda,\lambda)$ 
 &583.91&146.58&1209.14&164.96&2478.51&294.53&5084.68&462.34\\
\end{tabular}}

\resizebox{\linewidth}{!}{
\begin{tabular}[]{l || c | c || c | c || c | c || c | c }
\multirow{2}{*}{\textbf{Algorithms} }&  \multicolumn{2}{c||}{$n=1024$} &  
\multicolumn{2}{c||}{$n=2048$} &  \multicolumn{2}{c||}{$n=4096$} &  
\multicolumn{2}{c}{$n=8192$}\\ \cline{2-9}
& Mean & Std. dev.& Mean & Std. dev.& Mean & Std. dev.& 
Mean& Std. dev.\\  \hline
 (1+1)~EA & 17267.39 & 3653.38 & 38636.71 & 6966.43 & 84286.18 & 13563.25 & 
186012.84 & 
28660.69 \\ 
 (2+1)~GA & 14715.00 & 2876.00 & 32843.00 & 5574.70 & 71346.00 & 10810.00 & 
156800.00 & 
23357.00 \\ 
Greedy (2+1)~GA & 14553.66 & 2892.29 & 32667.89 & 6075.17 & 71149.61 & 
11990.71 & 154354.66 & 23250.14 \\
 (5+1)~GA & 13538.00 & 2436.00 & 29907.00 & 4909.60 & 65136.00 & 9758.00 & 
139590.00 & 18622.00 \\ 
 Sudholt`s (2+1)~GA$_{1/n}$ &10962.58 & 1960.08 & 24324.46 & 4543.38 & 51708.68 
& 8772.04 & 108990.46 & 
12729.60 \\ 
 Sudholt`s (2+1)~GA$_{opt}$  &10372.87 & 1545.92 & 22335.50 & 3117.32 & 
47913.78 & 6094.81 & 102614.39 & 
12261.37 \\ 
 (2+1)$_{S}$ GA&
14028.86 & 2852.73 & 31403.85 & 5935.81 & 68957.18 & 11905.57 & 151635.40 & 
25489.27 \\ 
 Sudholt's greedy selection + greedy XO (2+1)~GA$_{1/n}$&
9206.59 & 1713.77 & 20446.31 & 3691.36 & 44677.26 & 7311.26 & 96525.02 & 
13997.87 \\ 
 Sudholt`s greedy selection (2+1)~GA$_{1/n}$ & 
10640.57 & 1777.29 & 23035.28 & 3624.30 & 49857.03 & 7123.71 & 108087.00 & 
14881.14 \\
 Sudholt's (2+1)GA diverse 
crossover$_{opt}$ & 9132.63 & 149.92 & 20098.44 &3171.93 &43815.65 &6334.92 & 
93581.99 &12396.22 \\
 Self-adjusting $1+(\lambda,\lambda)$ 
 &10324.62&695.95&20951.38&1157.73&42216.53&1862.68&85028.97&2703.08 \\

% 109563.89 & 15262.57 \\ 
\end{tabular}}
 \end{table*}
\end{landscape}

\restoregeometry
%}

% \begin{table*}[t]
% %  \caption{Statistics for the experimental results of Fig.~\ref{fig:ave} and 
% % Fig.~\ref{fig:greed}}
% \end{table*}
\section{Conclusion}
The question of whether genetic algorithms can hillclimb faster than 
mutation-only algorithms is a 
long standing one. 
On one hand, in his pioneering book, Rechenberg
had given preliminary experimental evidence that crossover may speed up the runtime
of population based EAs for generalised \onemax \cite{Rechenberg1973}.
On the other hand, further experiments  
suggested that genetic algorithms were slower hillclimbers than the \ea \cite{Jansen2005,MitchellHollandForrest94}. 
In recent years it has been rigorously shown that crossover and mutation can outperform  
algorithms using only mutation.
Firstly, a new theory-driven GA called (1+($\lambda$,$\lambda$))~GA has been 
shown to 
be asymptotically faster for hillclimbing the \onemax function than any unbiased 
mutation-only EA \cite{DoerrDoerrEbel2015}.
Secondly, it has been shown how standard ($\mu$+$\lambda$)~GAs are 
twice as fast as
their standard bit mutation-only counterparts for \onemax as long as diversity 
is enforced 
through environmental selection \cite{Sudholt2015}.
\begin{figure}[t]
\caption{Average runtime gain of the \muGA versus the (1+1)~EA for different 
population sizes, errorbars show the standard deviation normalised by the 
average runtime of the (1+1) EA.}
\label{fig:EA}
\includegraphics[width=\textwidth]{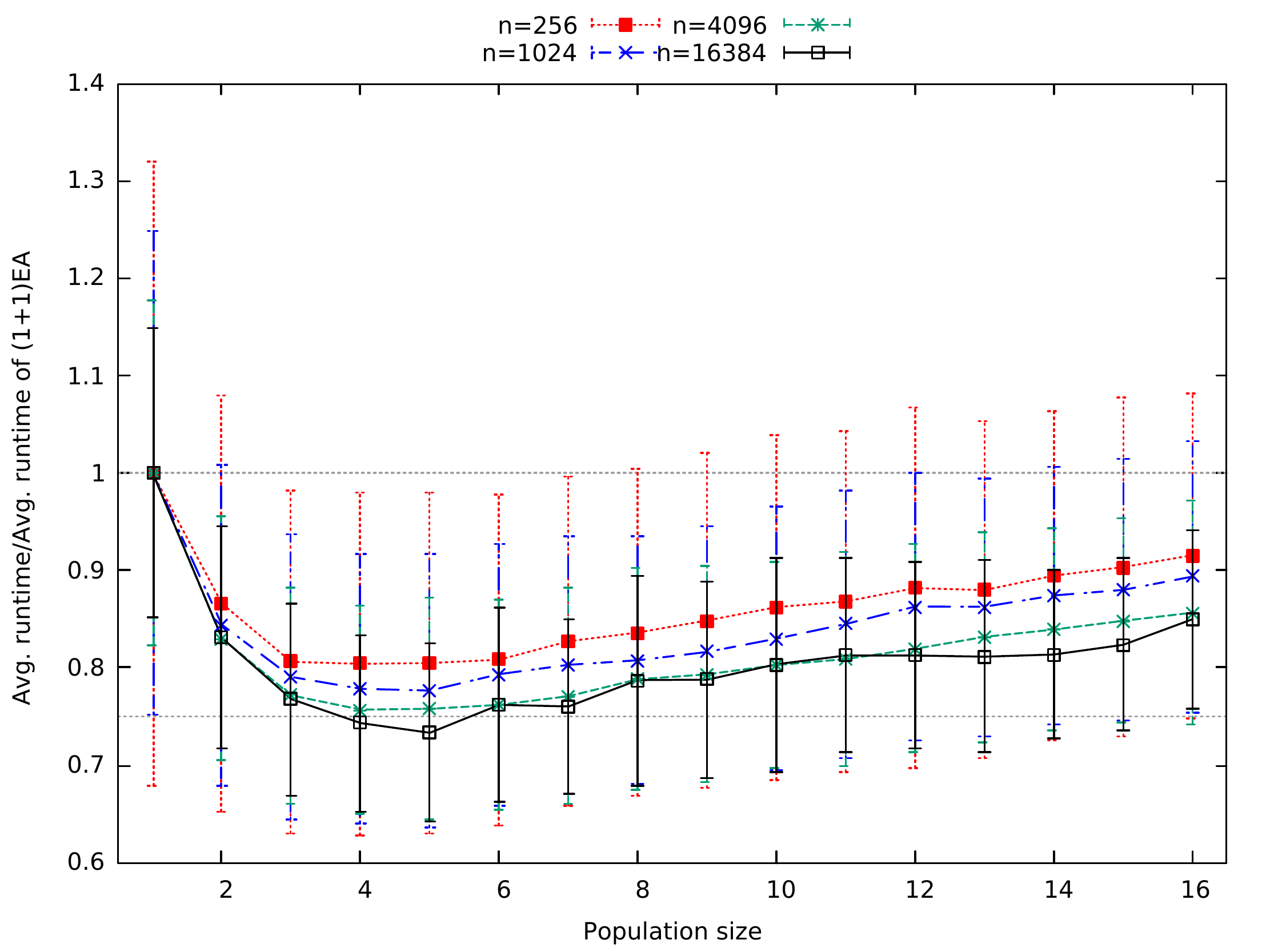}
 \end{figure}
\begin{figure}[t]
\caption{Average runtime gain of the (5+1)~GA for various mutation rates versus 
the standard 
$1/n$ mutation rate, errorbars show the standard deviation normalised by the 
average runtime for $1/n$ mutation rate.}
\label{fig:chi}
\includegraphics[width=\textwidth]{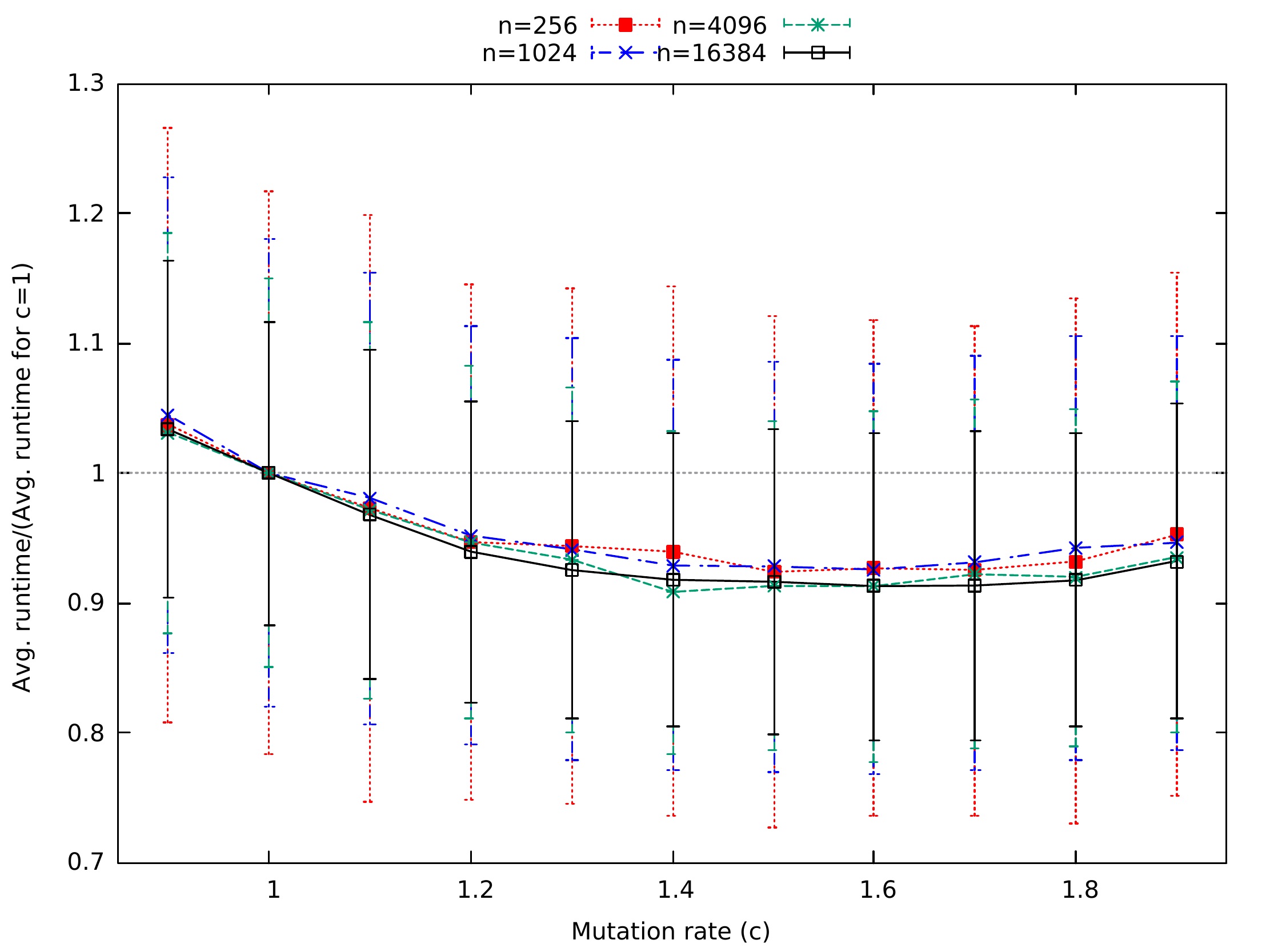}
\end{figure}

In this paper we have rigorously proven that standard steady-state GAs with 
 $\mu \geq 3$ and $\mu=o(\log{n}/\log{\log{n}})$ are at least 25\% faster than 
all unbiased standard bit mutation-based EAs  with static
mutation rate for \onemax even if no diversity is enforced. The Markov Chain 
framework we used to achieve the upper bounds on the runtimes
should be general enough to allow future analyses of more complicated GAs, for instance with greater offspring population sizes
 or more sophisticated crossover operators. 
 A limitation of the approach is that it applies to classes of problems that have plateaus of equal fitness.
 Hence, for functions where each genotype has a different  fitness value our approach would not apply.
 An open question is whether the limitation is inherent to our framework or whether it is crossover that would not help steady-state EAs at all on such fitness landscapes.

Our results also explain that populations are useful not only in the exploration phase of the optimization, but also to improve
exploitation during the hillclimbing phases. In particular, larger population sizes increase the probability of creating and maintaining
diversity in the population. This diversity can then be exploited by the crossover operator. 
Recent results had already shown how the interplay between mutation and crossover may allow the emergence of diversity, which in turn 
allows to escape plateaus of local optima more efficiently compared to relying on mutation alone \cite{Dang2016b}.
Our work sheds further light on the picture by showing that populations, crossover and mutation together, 
not only  may escape optima more efficiently, but may be more effective also in the exploitation phase.

Another additional insight gained from the analysis is that the standard mutation rate $1/n$ may not be optimal for the \muGA on \onemax.
This result is also in line with, and nicely complements, other recent findings concerning steady state GAs.
For escaping plateaus of local optima it has been recently shown that increasing the mutation rate above the standard $1/n$ rate 
leads to smaller upper bounds on escaping times \cite{Dang2016b}. However, when jumping large low-fitness valleys, mutation rates of about $2.6/n$ seem to be optimal static rates
(see the experiment section in \cite{Arxiv2016b, JumpTEVC}). 
For \onemax lower mutation rates seem to be optimal static rates, but still considerably larger than
the standard $1/n$ rate.

New interesting questions for further work have spawned. 
Concerning population sizes an open problem is to rigorously prove whether the optimal size grows with the problem size and at what rate.
Also determining the optimal mutation rate remains an open problem. 
While our theoretical analysis delivers the best upper bound on the runtime with a mutation rate of about 1.3/$n$, experiments suggest a larger optimal mutation rate.
Interestingly, this experimental rate is very similar to the optimal mutation 
rate (i.e., approximately 1.618/$n$) of the ($\mu$+1)~GA   with enforced 
diversity proven in \cite{Sudholt2015}. The 
benefits of higher than standard mutation rates in elitist algorithms is a 
topic 
that is gaining increasing interest~\cite{RankBasedGA,Donya2017,FastGA}.

Further improvements may be achieved by dynamically adapting the 
population size and mutation rate during the run. Advantages, in this sense, have been shown for the  (1+($\lambda$,$\lambda$))~GA 
by adapting the population size \cite{DoerrAdaptiveLambda} and for single 
individual algorithms by adapting the mutation rate \cite{DoerrPPSN2016, 
HyperGecco2017}.    
Generalising the results to larger classes of hillclimbing problems is intriguing. 
In particular, proving whether speed ups of the \muGA compared to the (1+1)~EA  are also achieved for royal road functions would 
give a definitive answer to a long standing question \cite{MitchellHollandForrest94}. 
Analyses for larger problem classes such as linear functions and classical combinatorial optimisation problems would lead to further insights.
Yet another natural question is how the \muGA hillclimbing capabilities compare to ($\mu$+$\lambda$)~GAs and generational GAs.

\bibliography{IEEEabrv,document}
\bibliographystyle{IEEEtran}
\end{document}